\documentclass[sigconf]{acmart}

\renewcommand\footnotetextcopyrightpermission[1]{} 
\usepackage{balance}
\usepackage{booktabs} 

\usepackage{color}
\usepackage{graphicx}
\usepackage{algorithm}
\usepackage{algpseudocode}
\usepackage{natbib}
\usepackage{amsmath} 
\usepackage{mathtools}       
\usepackage{stmaryrd}        
\usepackage{bbm}              

\usepackage{xr}             
\providecommand{\argmin}{\operatornamewithlimits{argmin}} 
\providecommand{\limsup}{\operatornamewithlimits{limsup}} 
\providecommand{\liminf}{\operatornamewithlimits{liminf}} 
\DeclareMathOperator{\Tr}{Tr}     
\DeclareMathOperator{\Var}{Var}   
\DeclareMathOperator{\Cond}{Cond} 
\DeclareMathOperator{\diag}{diag} 
\providecommand{\R}{\mathbb{R}} 
\providecommand{\E}{\mathbb{E}} 
\providecommand{\T}{\mathrm{T}} 
\providecommand{\ind}[1]{\mathbb{I}\left\{#1\right\}} 
\renewcommand{\geq}{\geqslant} 
\renewcommand{\leq}{\leqslant} 

\DeclarePairedDelimiterX{\inner}[2]{\langle}{\rangle}{#1, #2}
\DeclarePairedDelimiter{\norm}{\lVert}{\rVert}
\DeclarePairedDelimiter{\abs}{\lvert}{\rvert}

\usepackage{cleveref}

\usepackage{amsthm}
\theoremstyle{definition}
\newtheorem{theorem}{Theorem}
\newtheorem{proposition}[theorem]{Proposition}
\newtheorem{corollary}[theorem]{Corollary}
\newtheorem{lemma}[theorem]{Lemma}
\newtheorem{definition}[]{Definition}

\usepackage{ifthen}
\newcommand{\markupdraft}[2]{
    \ifthenelse{\equal{#1}{display}}{#2}{}
    \ifthenelse{\equal{#1}{color}}{\color{#2}}{}
}

\newcommand{\newcolored}[3][]{{\markupdraft{color}{#2}#3}
    \ifthenelse{\equal{#1}{}}{}{\markupdraft{display}{{\color{yellow!70!black}[#1]}}}}

\acmConference[]{}{}{}

\newcommand{\indup}{\mathbb{I}_{\uparrow}}
\newcommand{\inddown}{\mathbb{I}_{\downarrow}}
\newcommand{\inds}{\mathbb{I}_{s}}
\newcommand{\indl}{\mathbb{I}_{\ell}}

\newcommand{\aup}{\alpha_{\uparrow}}
\newcommand{\adown}{\alpha_{\downarrow}}

\begin{document}

\title[Convergence  Rate of the (1+1)-Evolution Strategy on Convex Quadratic Functions]{Convergence  Rate of the (1+1)-Evolution Strategy with Success-Based Step-Size Adaptation on Convex Quadratic Functions}

\author{Daiki Morinaga}
\affiliation{%
  \institution{University of Tsukuba \& RIKEN AIP}
  \streetaddress{1-1-1 Tennodai}
  \city{Tsukuba}
  \state{Ibaraki}
  \country{Japan}
}
\email{morinaga@bbo.cs.tsukuba.ac.jp}

\author{Kazuto Fukuchi}
\affiliation{%
  \institution{University of Tsukuba \& RIKEN AIP}
  \streetaddress{1-1-1 Tennodai}
  \city{Tsukuba}
  \state{Ibaraki}
  \country{Japan}
}
\email{fukuchi@cs.tsukuba.ac.jp}

\author{Jun Sakuma}
\affiliation{%
  \institution{University of Tsukuba \& RIKEN AIP}
  \streetaddress{1-1-1 Tennodai}
  \city{Tsukuba}
  \state{Ibaraki}
  \country{Japan}
}
\email{jun@cs.tsukuba.ac.jp}

\author{Youhei Akimoto}
\affiliation{%
  \institution{University of Tsukuba \& RIKEN AIP}
  \streetaddress{1-1-1 Tennodai}
  \city{Tsukuba}
  \state{Ibaraki}
  \country{Japan}
}
\email{akimoto@cs.tsukuba.ac.jp}

%


\begin{abstract}
      The (1+1)-evolution strategy (ES) with success-based step-size adaptation is analyzed on a general convex quadratic function and its monotone transformation, that is, $f(x) = g((x - x^*)^\T H (x - x^*))$, where $g:\R\to\R$ is a strictly increasing function, $H$ is a positive-definite symmetric matrix, and $x^* \in \R^d$ is the optimal solution of $f$.
      The convergence rate, that is, the decrease rate of the distance from a search point $m_t$ to the optimal solution $x^*$, is proven to be in $O(\exp( - L / \Tr(H) ))$, where $L$ is the smallest eigenvalue of $H$ and $\Tr(H)$ is the trace of $H$.
      This result generalizes the known rate of $O(\exp(- 1/d ))$ for the case of $H = I_{d}$ ($I_d$ is the identity matrix of dimension $d$) and $O(\exp(- 1/ (d\cdot\xi) ))$ for the case of $H = \diag(\xi \cdot I_{d/2}, I_{d/2})$.
      To the best of our knowledge, this is the first study in which the convergence rate of the (1+1)-ES is derived explicitly and rigorously on a general convex quadratic function, which depicts the impact of the distribution of the eigenvalues in the Hessian $H$ on the optimization and not only the impact of the condition number of $H$.
\end{abstract}

%
%
%
%

\maketitle

\sloppy

\section{Introduction}

\paragraph{Background}

    The evolution strategy (ES) is one of the most competitive classes of randomized algorithms in the field of continuous black-box optimization (BBO), where the objective function $f:\R^d \to \R$ can only be accessed through a black-box query $x \mapsto f(x)$, and neither the gradient $\nabla f$ nor the characteristic constants such as Lipschitz constants are available.

    In particular, variants of the covariance matrix adaptation evolution strategy (CMA-ES) \cite{hansen2001completely, hansen2003reducing, hansen2014principled} demonstrate the prominent efficiency in overcoming a variety of BBO difficulties in both real-world and benchmark problems
    \cite{rios2013derivative, uhlendorf2012long, kriest2017calibrating, geijtenbeek2013flexible, ha2018recurrent, volz2018evolving, dong2019efficient, fujii2018cma, varelas2020comparative}.
  Many studies have been conducted to improve the competitiveness of CMA-ES \cite{jastrebski2006improving, akimoto2014comparison, akimoto2020diagonal, auger2004ls}.
  However, evolutionary approaches, including the above-mentioned references, are often developed based on the empirical evaluation on sets of benchmark problems, and their theoretical guarantee is yet to be developed sufficiently.


As a step toward the convergence guarantee of the state-of-the-art CMA-ES variants, in this study, we analyzed a variant of ESs, the so-called (1+1)-ES with success-based step-size adaptation, which is the oldest variant of ES and was originally proposed by Rechenberg  in 1973 \cite{rechenberg1973evolution}.
    Although the (1+1)-ES is drastically simpler than other CMA-ES variants, especially because of the lack of covariance matrix adaptation mechanisms, they share several core features, such as randomness and heuristic step-size adaptation mechanisms.
    These features are surely the key to success but make it impossible to guarantee that the step size stays in a bounded range where a sufficient decrease in the objective function value in one step is guaranteed, which is the basis of the analysis of deterministic mathematical programming approaches \cite{devolder2013first, nesterov2013gradient}.
    There is always a (possibly small) probability that the step size will go out of the desired range during the optimization.
    This makes it difficult to analyze the convergence properties of ES variants, including CMA-ES variants.

 We study the convergence rate of the (1+1)-ES on convex quadratic functions and their monotone transformation, that is, $f(x) = g((x-x^*)^\T H (x - x^*) / 2)$, where $g: \R \to \R$ is a strictly increasing function, $H$ is a positive definite symmetric matrix, and $x^* \in \R^d$ is the optimal point.
    In particular, we are interested in obtaining the dependency of the convergence rate on the search space dimension $d$ and the condition number $\Cond(H)$ of the Hessian matrix, that is, the ratio between the greatest and smallest eigenvalues of $H$.
    Because an arbitrary twice continuously differentiable function can be approximated by a convex quadratic function around its local optimal points, the asymptotic analysis on a general convex quadratic function is ubiquitous in the local convergence analysis on a broader class of functions.
    Previous studies \cite{auger2013linear, morinaga2019generalized} showed that the (1+1)-ES converges linearly toward the global optimum on function classes including general convex quadratic functions; however, the convergence speed has not been given explicitly.

\paragraph{Related Work}


Akimoto et al.~\cite{akimoto2018drift} proved that the expected runtime of the (1+1)-ES until it finds an $\epsilon$-neighbor of the optimum is in $\Theta(d \cdot \log(1/\epsilon))$ on the spherical function $f: x \mapsto g(\norm{x}^2)$.
%

J\"{a}gersk\"{u}pper~\cite{jagerskupper2003analysis, jagerskupper20061+, jagerskupper2007algorithmic} analyzed a spherical function and a specific convex quadratic function $f(x_1,\ldots,x_d) = \xi\sum_{i=1}^{d/2} x_i^2 + \sum_{j=d/2+1}^d x_j^2$ (note $\xi = \Cond(H)$ in our notation), where $1/\xi\rightarrow 0$ as $d\to \infty$. The number of function evaluations to halve the function value is proven to be in $\Theta(d \cdot \xi)$ with an overwhelming probability.
Loosely speaking, these results translate to the convergence rate of $\Theta\left(\exp\left(-\frac{1}{d \cdot \Cond(H)}\right)\right)$.

Linear convergence of the (1+1)-ES was proven on a wider class of functions.
Auger and Hansen~\cite{auger2013linear} performed the Markov chain analysis of the (1+1)-ES on positively homogeneous functions with smooth levelsets, where a positively homogeneous function is a class of functions, all levelsets of which are similar in shape.
Morinaga and Akimoto~\cite{morinaga2019generalized} extended the scope of the analysis of \cite{akimoto2018drift} and showed the linear convergence of the (1+1)-ES, including strongly convex and Lipschitz smooth functions and positively homogeneous functions.
However, in these studies, the convergence rate was not explicitly derived, and hence, its dependency on $d$ and $\Cond(H)$ was not revealed.

Glasmachers~\cite{glasmachers2020global} showed the convergence of the (1+1)-ES to a stationary point, even on a broader class of functions. Because the objective of \cite{glasmachers2020global} was not to show linear convergence, the convergence was not guaranteed to be linear.

To sum up, there is no prominent result that explicitly estimates the impact of an ill-conditioned problem on the convergence rate of the (1+1)-ES on a general convex quadratic function.
Nevertheless , for the specific convex quadratic function in \cite{jagerskupper20061+}, there are clues that the convergence rate of the (1+1)-ES is in $\Theta\left(\exp\left(-\frac{1}{d \cdot \Cond(H)}\right)\right)$.
The objective of the current study is to extend the scope of the latter analysis to a general convex quadratic function.

\paragraph{Contribution}

We prove that the convergence rate of the (1+1)-ES on a convex quadratic function and its monotone transformation is in $O\left(\exp\left(-\frac{L}{\Tr(H)}\right)\right)$, where $L$ is the smallest eigenvalue of a positive definite Hessian $H$ and $\Tr(H)$ is the trace of $H$.
Because $\frac{L}{\Tr(H)} \geq \frac{1}{d \cdot \Cond(H)}$ for any positive definite symmetric $H$,
the convergence rate is indicated to be in $O\left(\exp\left(-\frac{1}{d \cdot \Cond(H)}\right)\right)$.

The order $O\left(\exp\left(-\frac{L}{\Tr(H)}\right)\right)$ is consistent with the upper bound obtained for a specific $H$ in previous studies \cite{jagerskupper2003analysis, jagerskupper2007algorithmic, jagerskupper20061+, akimoto2018drift}.
Our result is more rigorous than those of the previous studies in that the convergence rate reduces not only as the condition number increases but also as the distribution of large eigenvalues of the Hessian becomes heavier.
In addition, we provide a lower bound of the convergence rate, which is in $\Omega\left(\exp\left(-\frac{\Cond(H)}{d}\right)\right)$.
This lower bound is the first result that generalizes the $\Omega\left(\exp\left(-\frac{1}{d}\right)\right)$ lower bounds suggested for some specific convex quadratic functions in existing studies \cite{akimoto2018drift, jagerskupper2003analysis, jagerskupper20061+, jagerskupper2007algorithmic} into general convex quadratic functions, although it is looser with respect to $\Cond(H)$.
In the case of a spherical function, where $H$ is the identity matrix, the upper and lower bounds match, and our results above are consistent with the consequence of \cite{akimoto2018drift}.

%
%

\paragraph{Organization}

In \Cref{sec:algorithm_problem_formulation},
the algorithm to be studied is described in detail, and the concept of the algorithm design and important features are also discussed.
Next, the definition of the convergence rate and a technique to bound it are presented.
Then, the goal of the current research is formulated.
In \Cref{sec:lemmas},
we analyze the probability that the algorithm succeeds in advancing the optimization.
Based on this, the expected progress of the optimization is estimated. 
In \Cref{sec:upperconvergencerate},
we show that the convergence rate is in $O(\exp(-L/\Tr(H)))$ with a potential function approach proposed in \cite{akimoto2018drift, morinaga2019generalized}.
Furthermore, in \Cref{sec:lowerconvergencerate},
the convergence rate is revealed to be in $\Omega(\exp(-\Cond(H)/d))$.
In \Cref{sec:discussion},
we summarize the results and provide implications for a deeper understanding of our findings.
Finally, a perspective on future research is given.
Most of the proofs of technical lemmas are provided in the supplementary material.


\section{Formulation}\label{sec:algorithm_problem_formulation}

\subsection{Algorithm}\label{sec:algorithm}

We analyze the (1+1)-ES with success-based step-size adaptation, as described in \Cref{algo}.
The adaptation mechanism of the step size $\sigma$ is a generalized variant of the well-known $1/5$-success rule, first presented by Rechenberg \cite{rechenberg1973evolution} and then simplified by Kern et al.~\cite{kern2004learning}.

\begin{algorithm}[t]
  \caption{(1+1)-ES with success-based $\sigma$-adaptation}
  \label{algo}
  \begin{algorithmic}[1]
    \State \textbf{input} $m_0 \in \mathbb{R}^d$, $\sigma_0 > 0$, $f: \R^d \to \R$
    \State \textbf{parameter} $\aup > 1$, $\adown < 1$
    \For {$t = 0,1,\dots$, \textit{until the stopping criterion is met}}
    \State $x_t \sim m_t + \sigma_t \cdot z_t$, where $z_t \sim \mathcal{N}(0,  I)$
    \If {$f\big(x_t\big) \leq f\big(m_t\big)$}
    \State $m_{t+1} \leftarrow x_t$
    \State $\sigma_{t+1} \leftarrow \sigma_t \cdot \aup$
    \Else
    \State $m_{t+1} \leftarrow m_t$
    \State $\sigma_{t+1} \leftarrow \sigma_t \cdot \adown$
    \EndIf
    \EndFor
  \end{algorithmic}
\end{algorithm}

The main components of this algorithm are greedy selection of a candidate solution and step-size adaptation.
For each iteration $t \geq 0$, a candidate $x_t$ is sampled from the isotropic Gaussian distribution $\mathcal{N}(m_t, \sigma_t^2 \cdot I)$ (Line~4), the objective values of the candidate and current solutions (Lines ~5 and 8) are compared, the current solution is updated with the candidate solution  if the candidate solution is not worse than the current one (Lines 6, 9), and the step size is adapted in response to the result of the comparison (Lines 7 and 10).

Step-size adaptation is designed to maintain the probability of sampling a candidate solution better than or equally good as $p_\mathrm{target}$ , which is defined as
\begin{equation}
  p_\mathrm{target} := \frac{\log(1/\adown)}{\log(\aup / \adown)} \enspace.
\end{equation}
%
If the probability to obtain a better solution is always $p_\mathrm{target}$, the expected step size $\sigma_t$ keeps the same value, and if not, the algorithm attempts to control the probability by increasing or decreasing the step size $\sigma_t$.
In fact, if $f$ is continuously differentiable, the probability of $f(x_t)\leq f(m_t)$ can get arbitrarily closer to $1/2$ by taking $\sigma_t\rightarrow 0$ for any noncritical $m_t$, that is, $\nabla f(m_t) \neq 0$ \cite{glasmachers2020global}.

A mathematical model of the (1+1)-ES is defined as follows.
Hereafter, $\ind{\cdot}$ denotes the indicator function.
\begin{definition}\label{def:algorithm}
Let $\Theta = \R^{d+1}$ be the state space, and a state of \Cref{algo} at iteration $t$ is defined as $\theta_t = (m_t, \log(\sigma_t))$.
Let $\{z_t\}_{t \geq 0}$ be the sequence of independent and $\mathcal{N}(0, I)$-distributed random vectors.
For a measurable function $f : \R^d \to \R$, let $\theta_0 = (m_0, \log(\sigma_0))$ and $\theta_{t+1} = \theta_{t} + \mathcal{G}(\theta_t, z_t; f)$, where
\begin{equation}
  \begin{split}
    \mathcal{G}(\theta, z; f) := (\sigma \cdot z, \log(\aup)) \cdot &\ind{f(m + \sigma \cdot z) \leq f(m)}\\
    + (0, \log(\adown)) \cdot &\ind{f(m + \sigma \cdot z) > f(m)} \enspace.
  \end{split}
\end{equation}
Let $\{\mathcal{F}_t\}_{t \geq 0}$ be the natural filtration of $\{\theta_t\}_{t \geq 0}$.
We write
\begin{equation}
\{(\theta_t, \mathcal{F}_{t})\}_{t\geq 0} = \texttt{ES}(f, \theta_0, \{z_t\}_{t \geq 0}) \enspace.
\end{equation}
\end{definition}

The state sequence defined in \Cref{def:algorithm} is invariant to any strictly increasing transformation of the objective function and to any translation with a corresponding translation of the initial search point. This is formally stated in the following proposition. Its proof is provided in \Cref{apdx:prop:invariance}.
\begin{proposition}\label{prop:invariance}
  Given $\{z_t\}_{t \geq 0}$ and $\theta_0 \in \Theta$, the following hold.
  \begin{enumerate}
  \item Let $g:\R \to \R$ be an arbitrary strictly increasing function, i.e., $g(x) < g(y) \Leftrightarrow x < y$. Subsequently, $\tilde{\theta}_t = \theta_t$ for all $t \geq 0$ for $\{(\theta_t, \mathcal{F}_{t})\}_{t\geq 0} = \texttt{ES}(f, \theta_0, \{z_t\}_{t \geq 0})$ and $\{(\tilde{\theta}_t, \tilde{\mathcal{F}}_{t})\}_{t\geq 0} = \texttt{ES}(g \circ f, \theta_0, \{z_t\}_{t \geq 0})$.
  \item Let $T: x \mapsto x - x^*$ for an arbitrary $x^* \in \R$ and define $S_T: (m, \log(\sigma)) \to (T^{-1}(m), \log(\sigma))$. Subsequently, $\tilde{\theta}_t = S_T(\theta_t)$ for all $t \geq 0$,
    for $\{(\theta_t, \mathcal{F}_{t})\}_{t\geq 0} = \texttt{ES}(f, \theta_0, \{z_t\}_{t \geq 0})$ and $\{(\tilde{\theta}_t, \tilde{\mathcal{F}}_{t})\}_{t\geq 0} = \texttt{ES}(f \circ T, S_T(\theta_0), \{z_t\}_{t \geq 0})$.
  \end{enumerate}
\end{proposition}

\subsection{Convergence Rate}

Previous studies \cite{auger2013linear, jagerskupper2003analysis, jagerskupper20061+, jagerskupper2007algorithmic, akimoto2018drift, morinaga2019generalized, akimoto2020global} suggested that the (1+1)-ES converges not faster or slower than linear convergence over a broad class of functions including convex quadratic functions. Loosely speaking, linear convergence implies that the logarithm of the Euclidean distance $\log(\norm{m_t - x^*})$ from $m_t$ to the optimal point $x^*$ decreases by a constant.
We formally define the linear convergence of the (1+1)-ES as follows.
%
%
\begin{definition}\label{def:convergence}
  Let $f: \R^d \to \R$ be a measurable function with the global optimum at $x^*$, and $\{\theta_t\}_{t \geq 0}$ be the sequence of the state vectors defined in \Cref{def:algorithm} with $m_0 \neq x^*$.
  If there exists a constant $A>0$ satisfying
  \begin{align}
    \mathrm{Pr}\left[\lim_{t\rightarrow \infty} \frac{1}{t} \log \left( \frac{\norm{ m_t - x^* }}{\norm{m_0 - x^*}} \right) = -A \right] = 1
    \enspace,
  \end{align}
  then $\exp(-A)$ is called the \emph{convergence rate} of the (1+1)-ES on $f$.
  The \emph{upper convergence rate} $\exp(-A^\text{sup})$ and the \emph{lower convergence rate} $\exp(-A^\text{inf})$ are defined as constants satisfying
  \begin{align}
    &\mathrm{Pr}\left[\limsup_{t \to \infty} \frac{1}{t} \log \left( \frac{\norm{ m_t - x^* }}{\norm{m_0 - x^*}} \right) = -A^\text{sup} \right] = 1
    \enspace,
    \\
    &\mathrm{Pr}\left[\liminf_{t \to \infty} \frac{1}{t} \log \left( \frac{\norm{ m_t - x^* }}{\norm{m_0 - x^*}} \right) = -A^\text{inf} \right] = 1
    \enspace.
  \end{align}
\end{definition}
Note that $A^\text{sup}$ and $A^\text{inf}$ always exist in $\R_{\geq 0}$, whereas $A$ does not necessarily exist.
$A$ exists if and only if $A^\text{sup} = A^\text{inf}$. Otherwise, $\exp(-A^\text{inf}) < \exp(-A^\text{sup}) < 0$.
Our objective is to derive an upper bound of $\exp(-A^\text{sup})$ and a lower bound of $\exp(-A^\text{inf})$.

We bound the lower convergence rate and the upper convergence rate from below and above, respectively, using the following proposition that relies on the strong law of large numbers on martingales \cite{chow1967strong}.
Its proof is provided in \Cref{apdx:prop:convergence_rate}.

\begin{proposition}\label{prop:convergence_rate}
  Let $\{\mathcal{F}_t\}_{t \geq 0}$ be a filtration of a $\sigma$-algebra, and $\{X_t\}_{t \geq 0}$ be a Markov chain adapted to $\{\mathcal{F}_t\}_{t \geq 0}$.
  Consider the following conditions:
  \begin{enumerate}
  \item[C1] $\exists B > 0$ such that $\E[X_{t+1} - X_{t} \mid \mathcal{F}_t] \leq - B$ for all $t \geq 0$;
  \item[C2] $\exists C > 0$ such that $\E[X_{t+1} - X_{t} \mid \mathcal{F}_t] \geq - C$ for all $t \geq 0$;
  \item[C3] $\sum_{t=1}^{\infty}\Var[X_{t} \mid \mathcal{F}_{t-1}] / t^2 < \infty$.
  \end{enumerate}
  If C1 and C3 hold, then
  \begin{align}
    \Pr\left[\limsup_{t \to \infty} \frac{1}{t} (X_t - X_0) \leq - B \right] = 1 \label{eq:prop:upper}
    \enspace.
  \end{align}
  If C2 and C3 hold, then
  \begin{align}
    \Pr\left[\liminf_{t \to \infty} \frac{1}{t} (X_t - X_0) \geq - C \right] = 1\label{eq:prop:lower}
    \enspace.
  \end{align}
\end{proposition}

\subsection{Problem}\label{subsec:problemformulation}

We analyze the (1+1)-ES (\Cref{def:algorithm}) on convex quadratic functions and its monotone transformations, which are defined as follows:
\begin{definition}\label{def:problem}
  $\mathcal{Q}$ is a set of all functions $f:\R^d \to \R$ that can be expressed as $f(x) = g( (x - x^*)^\mathrm{T} H (x - x^*) )$, where $g: \R \to \R$ represents a strictly increasing function, $H \in \R^{d \times d}$ represents a positive definite symmetric matrix, and $x^* \in \R^d$ is the global optimal point. For simplicity, we call $H$ the Hessian matrix of $f$.
\end{definition}

The property of convex quadratic functions allows us to investigate $\log\left(\frac{f(m_t)}{f(m_0)}\right)$ instead of $\log\left(\frac{ \norm{m_t - x^*}} { \norm{m_0 - x^*} }\right)$ to derive the upper convergence rate and the lower convergence rate. This is advantageous because $f(m_t)$ is not increasing in $t$ with \Cref{algo}, whereas $\norm{m_t - x^*}$ may increase.
Moreover, \Cref{prop:invariance} shows that it is sufficient to work on $f: x \mapsto \frac12 x^\T H x$ to represent the analysis on $h \in \mathcal{Q}$.
It is formally stated in the following proposition, the proof of which is provided in \Cref{apdx:prop:f-norm}.

\begin{proposition}\label{prop:f-norm}
  Let $f \in \mathcal{Q}$ and $g$, $H$, $x^*$, as given in \Cref{def:problem}.
  Let $h: x \mapsto \frac12 x^\T H x$ and $S: (m, \log(\sigma)) \mapsto (m + x^*, \log(\sigma))$.
  We define $\{(\theta_t, \mathcal{F}_{t})\}_{t\geq 0} = \texttt{ES}(f, \theta_0, \{z_t\}_{t \geq 0})$ and $\{(\tilde{\theta}_t, \tilde{\mathcal{F}}_{t})\}_{t\geq 0} = \texttt{ES}(h, S(\theta_0), \{z_t\}_{t \geq 0})$.
  Next, for any $\theta_0 \in \Theta$, with probability one (note that the probability is taken for $\{z_t\}_{t \geq 0}$), we obtain
  \begin{align}
    \limsup_{t \to \infty} \frac{1}{t} \log \left( \frac{\norm{ m_t - x^* }}{\norm{m_0 - x^*}} \right)
    &=
      \limsup_{t \to \infty} \frac{1}{2 t} \log \left( \frac{h(\tilde{m}_t)}{h(\tilde{m}_0)} \right)
      \label{eq:limsup}
    \\
    \liminf_{t \to \infty} \frac{1}{t} \log \left( \frac{\norm{ m_t - x^* }}{\norm{m_0 - x^*}} \right)
    &=
      \liminf_{t \to \infty} \frac{1}{2 t} \log \left( \frac{h(\tilde{m}_t)}{h(\tilde{m}_0)} \right)
      \label{eq:liminf}
    \enspace.
  \end{align}
\end{proposition}

\section{Lemmas}\label{sec:lemmas}

In this section, we assume that $f: \R^d \to \R$ is a convex quadratic function $f(x) = \frac12 x^\T H x$, where $H$ is a positive definite symmetric matrix, the smallest and greatest eigenvalues of which are denoted by $L$ and $U$, respectively.
We investigate the expected one-step decrease in $\log(f(m_t))$, that is, the expectation of $\log(f(m_{t+1})) - \log(f(m_t))$,
  and the \emph{success probability}, probability of the event $f(x_t)\leq f(m_t)$.


Exploiting the assumption that $f$ is convex quadratic,
an upper bound of the expected one-step decrease in $\log(f(m_t))$ is derived in the form of a product of a quadratic form of $\sigma_t$ and the success probability.
The upper convergence rate bound is derived based on the following result.
Its proof is provided in \Cref{apdx:lemma:qualitygain}.
\begin{lemma}[]\label{lemma:qualitygain}
  Let $z$ be a random vector that follows the $d$-dimensional standard normal distribution.
  For any $m \in \R^d$ and $\sigma > 0$,
  \begin{multline}
    \E\left[\log\left( \frac{ f(m + \sigma z) }{ f(m) } \right) \ind{f(m + \sigma z) \leq f(m)}\right]
    \leq \\
    \frac{ \sigma \| \nabla f(m) \| }{ f(m) } \left(\frac{1}{4} \frac{\sigma Tr(H)}{ \| \nabla f(m) \|} - \frac{1}{\sqrt{2\pi}} \right) \cdot \Pr[f(m + \sigma z) \leq f(m) ]
    \enspace.
    \label{eq:qualitygain}
  \end{multline}
\end{lemma}


The following lemma is used to bound the variance of the one-step decrease in $\log(f(m_t))$, which will be used to prove condition C3 of \Cref{prop:convergence_rate}.
It is also used to lower-bound the expected one-step decrease in $\log(f(m_t))$.
Its proof is provided in \Cref{apdx:lemma:variancebound}.
\begin{lemma}[]\label{lemma:variancebound}
  Suppose $d > 3$ and let $z$ denote a random variable that follows the $d$-dimensional standard normal distribution.
  For any $m \in \R^d$ and $\sigma > 0$,
  \begin{multline}
    \E\left[\exp\left( \abs*{\log\left( \frac{ f(m + \sigma z) }{ f(m) } \right)} \ind{f(m + \sigma z) \leq f(m)} \right)\right]
    \\
    \leq
    1 + \frac{1}{d-3}\frac{U}{L} \enspace.
    \label{eq:lemma:variancebound}
  \end{multline}
\end{lemma}


If $\Tr(H^2)/ \Tr(H)^2$ is sufficiently small, which occurs if $\Cond(H)$ is bounded and $d \to \infty$ because $\Tr(H^2)/ \Tr(H)^2 \leq \Cond(H)^2 / d$, the success probability can be approximated by $\Phi\left( - \frac12 \frac{\sigma \Tr(H)}{\norm{\nabla f(m)}}\right)$ with the cumulative distribution function of the one-dimensional standard normal distribution.
It is formally stated in the following lemma, the proof of which is provided in \Cref{apdx:lemma:successprobability}.

\begin{lemma}[]\label{lemma:successprobability}
  Let $z$ denote a random vector that follows the $d$-dimensional standard normal distribution,
  and $\Phi(\cdot)$ denote the cumulative distribution function of the one-dimensional standard normal distribution.
  For any $m \in \R^d$, $\sigma > 0$, and $\epsilon > 0$,
  \begin{multline}
    \Phi\left( - \frac12 \frac{\sigma \Tr(H)}{\norm{\nabla f(m)}}\cdot(1+\epsilon)\right)
    - \frac{2}{\epsilon^2} \cdot \frac{\Tr(H^2)}{\Tr(H)^2}
    \\
    <
    \Pr\left[f(m + \sigma z) \leq f(m) \right]
    \\
    <
    \Phi\left( - \frac12 \frac{\sigma \Tr(H)}{\norm{\nabla f(m)}}\cdot(1-\epsilon)\right)
    + \frac{2}{\epsilon^2} \cdot \frac{\Tr(H^2)}{\Tr(H)^2}
    \enspace.
    \label{eq:lemma:successprobability}
  \end{multline}
\end{lemma}

As a corollary of \Cref{lemma:successprobability}, we can derive sufficient conditions on $\frac{\sigma\Tr(H)}{\norm{\nabla f(m)}}$ to upper-bound and lower-bound the success probability.
The proof is provided in \Cref{apdx:cor:successprobability}.
\begin{corollary}[]\label{cor:successprobability}
  Define $B_{H}^\mathrm{high} : \left(0, \frac12\right) \to \R_{> 0}$ as
  \begin{equation}
    B_{H}^\mathrm{high}(q) := \sup_{\sqrt{ \frac{ 4 }{1 - 2q} \cdot \frac{\Tr(H^2)}{\Tr(H)^2} } < \epsilon}
    \frac{ 2 \Phi^{-1} \left(1 - \left(q + \frac{2}{\epsilon^2} \cdot \frac{\Tr(H^2)}{\Tr(H)^2}\right) \right)}{ 1 + \epsilon} \enspace.
    \label{eq:cor:b_high}
  \end{equation}
  Subsequently,
  \begin{equation}
    \frac{\sigma \Tr(H)}{\norm{\nabla f(m)}}
    \leq B_{H}^\mathrm{high}(q)
    \Rightarrow
    \Pr\left[f(m + \sigma z) \leq f(m) \right] > q
    \enspace.
    \label{eq:cor:successprobability:large}
  \end{equation}
  $B_{H}^\mathrm{high}$ is right-continuous, strictly decreasing, and $B_{H}^\mathrm{high}(q) \leq 2 \Phi^{-1}(1-q)$ for all $q \in \left(0, \frac12\right)$.

  Suppose that $\Tr(H^2) < \Tr(H)^2 / 4$.
  We define $B_{H}^\mathrm{low}: \left( 2 \cdot \frac{\Tr(H^2)}{\Tr(H)^2} , \frac{1}{2}\right) \to \R_{> 0}$ as
  \begin{equation}
    B_{H}^\mathrm{low}(q) := \inf_{\sqrt{ \frac{ 2 }{q} \cdot \frac{\Tr(H^2)}{\Tr(H)^2} } < \epsilon < 1}
    \frac{ 2 \Phi^{-1} \left(1 - \left(q - \frac{2}{\epsilon^2} \cdot \frac{\Tr(H^2)}{\Tr(H)^2}\right) \right)}{ 1 - \epsilon} \enspace.
    \label{eq:cor:b_low}
  \end{equation}
  Subsequently,
  \begin{equation}
    \frac{\sigma \Tr(H)}{\norm{\nabla f(m)}}
    \geq B_{H}^\mathrm{low}(q)
    \Rightarrow
    \Pr\left[f(m + \sigma z) \leq f(m) \right] < q
    \enspace.
    \label{eq:cor:successprobability:small}
  \end{equation}
  $B_{H}^\mathrm{low}$ is left-continuous, strictly decreasing, and $B_{H}^\mathrm{low}(q) \geq 2 \Phi^{-1}(1-q)$ for all $q \in \left( 2 \cdot \frac{\Tr(H^2)}{\Tr(H)^2} , \frac{1}{2}\right)$.
\end{corollary}

\Cref{lemma:cases}, derived from \Cref{lemma:qualitygain} and \Cref{cor:successprobability}, splits the state space $\Theta$ of the (1+1)-ES into three distinct areas.
Case (i) corresponds to a situation in which $\sigma$ is so small that the success probability is close to $1/2$.
Case (ii) corresponds to a situation in which $\sigma$ is so large that the success probability is close to $0$.
In both cases, the expected one-step decrease in $\log(f(m_t))$ is close to zero.
Case (iii) corresponds to a situation in which $\sigma$ is in a reasonable range, and we can guarantee that the expected decrease in $\log(f(m_t))$ is sufficient.
The proof is provided in \Cref{apdx:lemma:cases}.
\begin{lemma}[]\label{lemma:cases}
  Suppose that
  \begin{equation}
    \frac{\Tr(H^2)}{\Tr(H)^2} < \frac18 \cdot \min \left\{ \Phi\left(\frac{1}{\sqrt{2\pi}}\right) - \frac12 , 1 - \Phi\left( \frac{3}{\sqrt{2\pi}}\right)\right\} \enspace.
    \label{eq:tr_cond}
  \end{equation}
  Let $B_H^\mathrm{high}$ and $B_H^\mathrm{low}$ be defined in \Cref{cor:successprobability}.
  Subsequently, there exists $q^\mathrm{low}$ satisfying
  \begin{gather}
    2 \cdot \frac{\Tr(H^2)}{\Tr(H)^2} < q^\mathrm{low} < \frac{1}{2} \enspace,
    \\
    \frac{4}{\sqrt{2\pi}} > B_H^\mathrm{low}(q^\mathrm{low}) \enspace.\label{eq:bhqlow}
  \end{gather}
  Fix such $q^\mathrm{low}$, and then, there exists $q^\mathrm{high}$ such that $q^\mathrm{low}<q^\mathrm{high}<1/2$.
  We define
  \begin{align}
    Q_H &= \sup \left\{ Q : B_H^\mathrm{high}(Q) \geq B_H^\mathrm{low}(q^\mathrm{low}) \right\} \enspace.
          \label{eq:qh}
  \end{align}
  Thus, $Q_H > 0$.
  Moreover, the following statements hold.
  \begin{enumerate}
    \renewcommand{\labelenumi}{(\roman{enumi})}
  \item If $\sigma < B_H^\mathrm{high}(q^\mathrm{high}) \cdot \sqrt{ 2 L f(m) } / \Tr(H)$,
    \begin{equation}
      \Pr\left[f(m + \sigma z) \leq f(m) \right] \geq q^\mathrm{high} \enspace. \label{eq:lemma:toosmall}
    \end{equation}

  \item If $\sigma > B_H^\mathrm{low}(q^\mathrm{low}) \cdot \norm{\nabla f(m)} / \Tr(H)$,
    \begin{equation}
      \Pr\left[f(m + \sigma z) \leq f(m) \right] \leq q^\mathrm{low} \enspace. \label{eq:lemma:toolarge}
    \end{equation}

  \item Otherwise,
    \begin{equation}
      \Pr\left[f(m + \sigma z) \leq f(m) \right] \geq Q_H
      \label{eq:lemma:reasonable}
    \end{equation}
    and
    \begin{multline}
      \E\left[\log\left( \frac{ f(m + \sigma z) }{ f(m) } \right) \ind{f(m + \sigma z) \leq f(m)}\right]
      \\
      \leq
      \frac{ L \cdot  B_H^\mathrm{high}(q^\mathrm{high}) }{ 2 \Tr(H) }  \left( B_H^\mathrm{low}(q^\mathrm{low}) - \frac{4}{\sqrt{2\pi}} \right) \cdot Q_H
      < 0
      \enspace.
      \label{eq:lemma:reasonable:qualitygain}
    \end{multline}
  \end{enumerate}
\end{lemma}

\section{Upper Convergence Rate Bound}\label{sec:upperconvergencerate}

    In this section, the upper convergence rate of the (1+1)-ES on $f \in \mathrm{Q}$ is upper-bounded.
%
    We proceed the derivation as follows.
    Owing to  \Cref{prop:f-norm}, we can assume, without loss of generality, that $f(x) = \frac12 x^\T H x$.
    First, in \Cref{subsec:potentialfunction}, we define the potential function $V: \Theta \to \R$ satisfying $\log\left(f(m)\right) \leq V(\theta)$.
    Let $\{(\theta_t, \mathcal{F}_t)\}_{t \geq 0} = \texttt{ES}(f, \theta_0, \{z_t\}_{t\geq 0})$ for $\theta_0 \in \Theta$.
    Next, we apply \Cref{prop:convergence_rate} with $X_t = V(\theta_t)$.
    Conditions C1 and C3 in \Cref{prop:convergence_rate} are derived in \Cref{subsec:lowerbound}.
    Subsequently, we obtain $B > 0$ such that
    \begin{equation}
      \Pr\left[\limsup_{t \to \infty} \frac1t \frac{V(\theta_t)}{V(\theta_0)} \leq B\right] = 1 \enspace.
    \end{equation}
    Because $\log\left(f(m)\right) \leq V(\theta)$, using \Cref{prop:f-norm}, we obtain the upper bound of the upper convergence rate.
    Finally, we evaluate the dependency of $B$ on $d$ and $\Cond(H)$ in \Cref{subsec:main}.



\subsection{Potential Function}\label{subsec:potentialfunction}

The potential function $V(\theta)$ on a convex quadratic function $f(x) = \frac12 x^\T H x$ is defined as follows.
This form of the potential function was first proposed in \cite{akimoto2018drift} for the sphere function and was generalized in \cite{morinaga2019generalized}. The following definition is specialized for a convex quadratic function:

\begin{definition}[Potential function]\label{def:potential}
  Let $f: x \mapsto \frac12 x^\T H x $ with a positive definite symmetric $H$.
  Let $L$ and $U$ be the smallest and greatest eigenvalues of $H$, respectively.
  The potential function $V(\theta)$ for the (1+1)-ES solving $f$ is defined as
  \begin{multline}
    V(\theta)
    = \log\left(f(m)\right)
    \\
    + v \cdot \log^+ \left(\frac{b_s\sqrt{Lf(m)}}{\Tr(H) \sigma}\right) +  v \cdot \log^+ \left(\frac{\Tr(H)\sigma}{b_\ell\sqrt{Lf(m)}}\right) \enspace
    \label{eq:potential}
  \end{multline}
  where $v \in (0, 2)$ and $0<b_s<b_\ell$ are constants, and $\log^+(x) := \log(x) \cdot \ind{x\geq 1}$ .
\end{definition}

The first term measures the main quantity that we would like to decrease. 
The second and third terms measure the progress of $\sigma$-adaptation when $\sigma$ is too small and too large, respectively.
See \cite{akimoto2018drift} and \cite{morinaga2019generalized} for a more detailed description.

We define $b_s$, $b_\ell$, and $v$ as follows.
Suppose that \eqref{eq:tr_cond} holds.
We choose $q^\mathrm{low}$ and $q^\mathrm{high}$ such that
\begin{gather}
  2 \cdot \frac{\Tr(H^2)}{\Tr(H)^2} < q^\mathrm{low} < p_\mathrm{target} < q^\mathrm{high} < \frac{1}{2}
  \enspace,
  \label{eq:q_condition1}
  \\
  \frac{4}{\sqrt{2\pi}} > B_H^\mathrm{low}(q^\mathrm{low})
  > \frac{\aup}{\adown} B_H^\mathrm{high}(q^\mathrm{high})
  \enspace.
  \label{eq:q_condition2}
\end{gather}
If $\aup$ and $\adown$ are set so that
\begin{align}
  \frac{4}{\sqrt{2\pi}} > B_H^\mathrm{low}(p_\mathrm{target}) \label{eq:ptarget_condition}
\end{align}
holds,
as we see in \Cref{lemma:cases}, we can find such $q^\mathrm{low}$.
Because $B_H^\mathrm{high}(q) < 2 \Phi^{-1}(1 - q) \to 0$ as $q \to 1/2$,
we can find a pair $(q^\mathrm{low}, q^\mathrm{high})$ satisfying the conditions above.
Then, we set $b_s$ and $b_\ell$ as follows
\begin{align}
  b_s &= \sqrt{ 2 } \cdot B_H^\mathrm{high}(q^\mathrm{high}) \cdot \aup
        \label{eq:b_s}
        \enspace,
  \\
  b_\ell &= \sqrt{2} \cdot B_H^\mathrm{low}(q^\mathrm{low}) \cdot \adown
           \label{eq:b_l}
           \enspace.
\end{align}
It is easy to see from \cref{eq:q_condition2,eq:b_s,eq:b_l} that $0 < b_s < b_\ell$.
Let
\begin{equation}
  w = \frac{ L \cdot  B_H^\mathrm{high}(q^\mathrm{high}) }{ 2 \Tr(H) } \left( \frac{4}{\sqrt{2\pi}} - B_H^\mathrm{low}(q^\mathrm{low})  \right) \cdot Q_H,
  \label{eq:w}
\end{equation}
so that the left-hand side of \eqref{eq:lemma:reasonable:qualitygain} is upper-bounded by $-w$.
Finally, we set $v$ as
\begin{equation}
  v = \min\left\{ \frac{w}{4 \log(\aup / \adown)} , 1 \right\}
  \enspace.
  \label{eq:v}
\end{equation}

\subsection{Expected Potential Decrease}\label{subsec:lowerbound}

The following three lemmas guarantee that the potential function $V(\theta_t)$ decreases sufficiently in any case of the step size in \Cref{lemma:cases}.
In the following lemmas, let $f(x) = \frac12 x^\T H x$ with a positive definite symmetric $H$ satisfying \eqref{eq:tr_cond}. Let $\{(\theta_t, \mathcal{F}_t)\}_{t \geq 0} = \texttt{ES}(f, \theta_0, \{z_t\}_{t \geq 0})$ be the state sequence of the (1+1)-ES solving $f$ with $\theta_0 \in \Theta$.
The potential function $V$ is defined in \Cref{def:potential}.

    The first lemma is for the case in which the step size is too small to expect a sufficient decrease in $\log(f(m_t))$.
    In this situation, however, we can lower-bound the success probability by $q^\mathrm{high}$ in light of \Cref{lemma:cases}.
    This leads to a sufficient expected decrease in the second term in \eqref{eq:potential} as $\sigma$ is increased by $\aup > 1$ with a probability no less than $q^\mathrm{high}$.
    The proof is provided in \Cref{apdx:lemma:too_small_sigma}.
    \begin{lemma}[Small-step-size case]\label{lemma:too_small_sigma}
      If $\sigma_t < \frac{b_s\sqrt{L}}{\aup\Tr(H)} \sqrt{f(m_t)}$, the following holds:
      \begin{multline}
        \E[V(\theta_{t+1}) - V(\theta_t) \mid \mathcal{F}_t]\\
        \leq - \min\left\{\frac{w}{4}, \log\left(\frac{\aup}{\adown}\right)\right\} \cdot \left(q^\mathrm{high} - p_\mathrm{target}  \right)
        \enspace.
              \label{eq:lemma:too_small_sigma}
      \end{multline}
    \end{lemma}


    The second lemma is for the case in which the step size is too large.
    In this situation, \Cref{lemma:cases} ensures that the success probability is no greater than $q^\mathrm{low}$.
    The expected decrease in $\log(f(m_t))$ can be arbitrarily small as the success probability is close to zero.
    However, because $\sigma$ is decreased by $\adown < 1$ with probability no less than $1 - q^\mathrm{low}$, the third term in \eqref{eq:potential} will be decreased sufficiently in expectation.
    The proof is provided in \Cref{apdx:lemma:too_large_sigma}.
    \begin{lemma}[Large-step-size case]\label{lemma:too_large_sigma}
      If $\sigma_t > \frac{b_\ell}{\sqrt{2}\adown\Tr(H)}\norm{\nabla f(m_t)}$, it holds
      \begin{multline}
        \E[V(\theta_{t+1}) - V(\theta_t) \mid \mathcal{F}_t]
        \\
        \leq - \min\left\{\frac{w}{4}, \log\left(\frac{\aup}{\adown}\right)\right\} \cdot \left( p_\mathrm{target} -  q^\mathrm{low}    \right)
          \enspace.
          \label{eq:lemma:too_large_sigma}
      \end{multline}
    \end{lemma}

    The third lemma is for the case in which the decrease in $\log(f(m_t))$ is sufficiently large.
    Using \Cref{lemma:cases}, we can guarantee a sufficient expected decrease in the first term in \eqref{eq:potential}.
    The proof is provided in \Cref{apdx:lemma:reasonable_sigma}.
    \begin{lemma}[Reasonable-step-size case]\label{lemma:reasonable_sigma}
      If $\frac{b_s\sqrt{L}}{\aup\Tr(H)} \sqrt{f(m_t)}\leq \sigma_t \leq \frac{b_\ell}{\sqrt{2}\adown\Tr(H)}\norm{\nabla f(m_t)}$, it holds
      \begin{equation}
        \E\left[V(\theta_{t+1}) - V(\theta_t)\mid\mathcal{F}_t\right]
        \leq - \frac{w}{4} 
        \enspace.
        \label{eq:lemma:reasonable_sigma}
      \end{equation}
    \end{lemma}


    The three lemmas above are used to satisfy Condition C1 in \Cref{prop:convergence_rate}.
    Condition C3 in \Cref{prop:convergence_rate} is also satisfied by $V(\theta)$, which is stated in the next lemma. Its proof is provided in \Cref{apdx:lemma:v_variance_bound}.
    \begin{lemma}[]\label{lemma:v_variance_bound}
        Suppose $d>3$.
        It holds $\sum_{t=1}^{\infty}\Var[V(\theta_{t}) \mid \mathcal{F}_{t-1}] / t^2 < \infty$.
    \end{lemma}

    \subsection{Main Theorem}\label{subsec:main}
    Finally, we attain the main result.
    The upper convergence rate of the (1+1)-ES on a general convex quadratic function is shown to be in $O_{d \to \infty}\left( \exp\left(- \min\left\{\frac{L}{\Tr(H)}, \log\left(\frac{\aup}{\adown}\right)\right\} \right)\right)$ if $\Cond(H)$ is bounded.

    \begin{theorem}[Upper convergence rate bound]\label{theorem:lower}
      Assume the objective function $f: \R^d \to \R$ satisfies the following: $f \in \mathcal{Q}$ defined in \Cref{def:problem}, $d \geq 4$, and the Hessian matrix $H$ of $f$ satisfies \eqref{eq:tr_cond}.

      Let $\{(\theta_t, \mathcal{F}_t)\}_{t \geq 0} = \texttt{ES}(f, \theta_0, \{z_t\}_{t \geq 0})$ be the state sequence of the (1+1)-ES solving $f$ with $\theta_0 \in \Theta$.
      Suppose that $\aup$ and $\adown$ satisfy \eqref{eq:ptarget_condition}.

      Let
      \begin{equation}
        B = \sup_{q^\mathrm{low}, q^\mathrm{high}}
        \min\left\{\frac{w}{4}, \log\left(\frac{\aup}{\adown}\right)\right\} \cdot \min\{p_\mathrm{target} -  q^\mathrm{low}, q^\mathrm{high} - p_\mathrm{target} \} \enspace,
        \label{eq:B}
      \end{equation}
      where $w$ is defined in \eqref{eq:w}, and $\sup_{q^\mathrm{low}, q^\mathrm{high}}$ is taken over pairs satisfying \eqref{eq:q_condition1} and \eqref{eq:q_condition2}.
      Subsequently, $\exp(-A^\text{sup}) \leq \exp(-B / 2) < 1$ for all $\theta_0 \in \Theta\setminus\{(x^*, \log(\sigma)): \log(\sigma) \in \R\}$, where $\exp(-A^\text{sup})$ is the upper convergence rate defined in \Cref{def:convergence}.

      Moreover,
      let $\mathcal{Q}_{d,\kappa} = \{f : \R^d \to \R \mid f \in \mathcal{Q} \wedge \Cond(H) \leq \kappa\}$.
      Then,
      \begin{equation}
        \lim_{d \to \infty}
        \inf_{f \in \mathcal{Q}_{d,\kappa}} \frac{ A^\text{sup}}{ \min\left\{ \frac{L}{\Tr(H)},  \log\left(\frac{\aup}{\adown}\right)\right\} } > 0 \enspace.
        \label{eq:order_B}
      \end{equation}
    \end{theorem}

    \begin{proof}[Proof of \Cref{theorem:lower}]
      Let $h(x) = \frac12 x^\T H x$ and $S: \theta \mapsto (m - x^*, \log(\sigma))$.
      Let $\{(\tilde{\theta}_t, \tilde{\mathcal{F}}_t)\}_{t \geq 0} = \texttt{ES}(h, \tilde{\theta}_0, \{z_t\}_{t \geq 0})$ be the state sequence of the (1+1)-ES solving $h$ with $\tilde{\theta}_0 = S(\theta_0) \in \Theta$.
      Next, considering \Cref{prop:f-norm}, the upper convergence rate of $\{\tilde{\theta}_t\}_{t \geq 0}$ is equal to
      \begin{equation}
      \limsup_{t \to \infty} \frac{1}{2t} \log\left(\frac{h(\tilde{\theta}_t)}{h(\tilde{\theta}_0)}\right)
    \end{equation}
    with probability one.
    Therefore, without loss of generality, we assume that $f(x) = \frac12 x^\T H x$ in the remaining proof.

        Under condition \eqref{eq:ptarget_condition}, we can find a pair $(q^\mathrm{low}, q^\mathrm{high})$ satisfying \cref{eq:q_condition1,eq:q_condition2}, as described in \Cref{subsec:potentialfunction}.
        Subsequently, from \Cref{lemma:cases}, it is easy to see that $w$ in \cref{eq:w} and, hence, $v$ in \cref{eq:v} are strictly positive.
        Moreover, because $q^\text{high} > p_\text{target}$ and $q^\text{low} < p_\text{target}$, we have $B > 0$ for $B$ defined in \cref{eq:B}.

      Let $X_t = V(\theta_t)$ in \Cref{prop:convergence_rate} with $V$ defined in \Cref{def:potential} with $b_s$, $b_\ell$, and $v$ defined in \cref{eq:b_s,eq:b_l,eq:v}, respectively.
        Condition C1 --- $\E[V(\theta_{t+1}) - V(\theta_t)\mid\mathcal{F}_t]\leq -B$ for all $t\geq 0$ --- is satisfied with $B$ defined in \cref{eq:B} in light of \Cref{lemma:too_small_sigma,lemma:too_large_sigma,lemma:reasonable_sigma}.
        Condition C3 --- $\sum_{t=1}^{\infty}\Var[V(\theta_{t}) \mid \mathcal{F}_{t-1}] / t^2 < \infty$ --- is satisfied for $d>3$ in light of \Cref{lemma:v_variance_bound}.
        Therefore, with probability one, we obtain that
      \begin{equation}
      \limsup_{t \to \infty} \frac{1}{t} \frac{V(\theta_t)}{V(\theta_0)} \leq - B \enspace.
    \end{equation}
    Because $\log(f(m)) \leq V(\theta)$ for all $\theta \in \Theta$, we have
      \begin{equation}
      \limsup_{t \to \infty} \frac{1}{2t} \log\left(\frac{f(\theta_t)}{f(\theta_0)}\right) \leq - \frac{B}{2} \enspace.
    \end{equation}
    Hence, we obtain $\exp(-A^\text{sup}) \leq \exp(- B / 2) < 1$.

        Finally, we prove \eqref{eq:order_B}.
        Let $q^\mathrm{low}$ and $q^\mathrm{high}$ be set so that
        \begin{gather}
          0 < q^\mathrm{low} < p_\mathrm{target} < q^\mathrm{high} < \frac{1}{2}
          \enspace,
          \\
          \sqrt{\frac{2}{\pi}} > \Phi^{-1}(1-q^\mathrm{low})
          > \frac{\aup}{\adown} \Phi^{-1}(1-q^\mathrm{high})
          \enspace.
        \end{gather}
        For a function $f \in \mathcal{Q}_{d,\kappa}$, we have
        \begin{equation}
          \frac{\Tr(H^2)}{\Tr(H)^2} \leq \frac{\Cond(H)}{d} \to 0 \text{ as } d\to\infty
          \enspace.
          \label{eq:tr_cond_asymptotic}
        \end{equation}
        Subsequently, for each $q \in (0, 1/2)$, $B_H^\text{low}(q) \to 2 \Phi^{-1}(1-q)$ and $B_H^\text{high}(q) \to 2 \Phi^{-1}(1-q)$.
        Therefore, for any $\kappa \geq 1$, there exists $D > 0$ such that all $f \in \mathcal{Q}_{d,\kappa}$ with $d \geq D$ satisfy \eqref{eq:q_condition1} and \eqref{eq:q_condition2}.
        Considering that $Q_H \to q^\text{low}$ in the limit of $\frac{\Tr(H^2)}{\Tr(H)^2} \to 0$, we have $w \in \Omega_{d\to\infty}\left( \frac{L}{\Tr(H)} \right)$.
        Finally, as $A^\text{sup} \geq B / 2$, we obtain \eqref{eq:order_B}. This completes the proof.
    \end{proof}

    Note that \eqref{eq:tr_cond} is the only condition that restricts the scope of the analysis in terms of the class of functions.
    Furthermore, in light of \eqref{eq:tr_cond_asymptotic}, condition \eqref{eq:tr_cond} is satisfied for any Hessian matrix $H$ with a bounded condition number if $d$ is sufficiently large.
    Therefore, \Cref{theorem:lower} asymptotically provides the upper bound of the upper convergence rate on the general convex quadratic function for sufficiently large $d$.

    We remark on the consequences.
    The hyper-parameters $\aup$ and $\adown$ are often set depending on the search space dimension $d$,
    or typically, chosen such that $\log(\aup / \adown)\in\Theta_{d \to \infty}(1/d)$.
    The theorem ensures that as long as $\log(\aup / \adown) \in \Omega_{d \to \infty}(1/d)$, the upper convergence rate is in $O_{d \to \infty}\left( \exp\left(- \frac{L}{\Tr(H)} \right)\right)$.
    In contrast, if $\log(\aup / \adown) \in o_{d \to \infty}(1/d)$ , we have $A^\text{sup} \in O_{d \to \infty}\left( \exp\left(- \log\left(\frac{\aup}{\adown}\right) \right)\right) = O_{d \to \infty}\left( \frac{\adown}{\aup}\right) = O_{d \to \infty}\left( \adown^{1/p_\mathrm{target}}\right)$.
    This is rather intuitive for the following reason. $\norm{m_t - x^*}$ does not converge faster than $\sigma_t$ because $\sigma_t$ needs to be proportional to $\norm{m_t - x^*}$ to produce a sufficient decrease. The speed of the decrease in $\sigma_t$ is $\adown$. Therefore, the upper convergence rate should not be smaller than $\adown$.

    Bounding the upper convergence rate with $\frac{L}{\Tr(H)}$ is more informative than bounding it with $\frac{1}{d \cdot \Cond(H)}$. As mentioned in the introduction, we have $\frac{L}{\Tr(H)} \geq \frac{1}{d \cdot \Cond(H)}$. Therefore, the bound $O_{d \to \infty}\left( \exp\left(- \frac{L}{\Tr(H)} \right)\right)$ immediately implies that $O_{d \to \infty}\left( \exp\left(- \frac{1}{d\cdot\Cond(H)} \right)\right)$.
    However, even for the same condition numbers $\Cond(H) = \xi \geq 1$, the bound with $\frac{L}{\Tr(H)}$ can be significantly different, depending on the distribution of the eigenvalues of $H$. For example, let us consider the following two situations:
    \begin{align}
      H_\text{cigar} &= \diag(\xi, \cdots, \xi, 1) && \Rightarrow  & \frac{L}{\Tr(H_\text{cigar})} &= \frac{1}{(d - 1) \xi + 1}, \\
      H_\text{discus} &= \diag(\xi, 1, \cdots, 1) && \Rightarrow  & \frac{L}{\Tr(H_\text{discus})} &= \frac{1}{\xi + (d - 1)}\enspace.
    \end{align}
    For $H_\text{cigar}$, we have $O_{d \to \infty}\left( \exp\left(- \frac{1}{d\cdot\xi} \right)\right)$, whereas for $H_\text{discus}$, we have $O_{d \to \infty}\left( \exp\left(- \frac{1}{d} \right)\right)$. In other words, if only a small portion of the axes are sensitive to the objective function value (i.e., directions corresponding to the eigenvalues of $\xi$), the upper convergence rate on the ill-conditioned ($\Cond(H) \gg 1$) convex quadratic function can be as good as the upper convergence rate on the spherical ($\Cond(H) = 1$) convex quadratic function.

\section{Lower Convergence Rate Bound}\label{sec:lowerconvergencerate}

The lower bound of the lower convergence rate, $\exp(-A^{\inf}) > 0$, is obtained immediately from \Cref{prop:convergence_rate} and \Cref{lemma:variancebound}.

\begin{theorem}[Lower convergence rate bound]\label{theorem:upper}
  Assume that the objective function $f: \R^d \to \R$ satisfies the following: $f \in \mathcal{Q}$ defined in \Cref{def:problem} and $d \geq 4$.

  Let $\{(\theta_t, \mathcal{F}_t)\}_{t \geq 0} = \texttt{ES}(f, \theta_0, \{z_t\}_{t \geq 0})$ be the state sequence of the (1+1)-ES solving $f$ with $\theta_0 \in \Theta$.

  Subsequently, the lower convergence rate (\Cref{def:convergence}) of the (1+1)-ES solving $f$ is lower-bounded as
  \begin{equation}
    \exp\left(- A^\text{inf}\right) \geq \exp\left(- \frac{\Cond(H)}{2(d-3)} \right)
  \end{equation}
  for all $\theta_0 \in \Theta\setminus\{(x^*, \log(\sigma)): \log(\sigma) \in \R\}$.

\end{theorem}

\begin{proof}[Proof of \Cref{theorem:upper}]
  As discussed in the proof of \Cref{theorem:lower}, we can assume without loss of generality that $f(x) = \frac12 x^\T H x$.

  We apply \Cref{prop:convergence_rate} with $X_t = \log f(m_t)$.
  The LHS of \eqref{eq:lemma:variancebound} is $\E[\exp(\abs{X_{t+1} - X_{t}}) \mid \mathcal{F}_t ]$.
  Using the fact that $x \leq \exp(x) - 1$ for all $x \geq 0$, we obtain
  $\E[\abs{X_{t+1} - X_{t}} \mid \mathcal{F}_t ] \leq \E[\exp(\abs{X_{t+1} - X_{t}}) \mid \mathcal{F}_t ] - 1 \leq \frac{1}{d-3}\frac{U}{L}$
  with \Cref{lemma:variancebound}. Note that $X_{t+1} - X_t = - \abs{X_{t+1} - X_t}$ as it is non-positive. Hence, we obtain C2 of \Cref{prop:convergence_rate} with $C = \frac{1}{d-3}\frac{U}{L}$.
  Using the fact that $x^2 \leq 2 (\exp(x) - x - 1) \leq 2 (\exp(x) - 1)$ for all $x \geq 0$, we obtain
  $\E[\abs{X_{t+1} - X_{t}}^2 \mid \mathcal{F}_t ] \leq \frac{2}{d-3}\frac{U}{L}$. Because $\Var[X_{t+1} \mid \mathcal{F}_t] \leq \E[\abs{X_{t+1} - X_{t}}^2 \mid \mathcal{F}_t ]$, we obtain $\Var[X_{t+1} \mid \mathcal{F}_t] \leq \frac{2}{d-3}\frac{U}{L}$. C3 in \Cref{prop:convergence_rate} is then satisfied.
  Therefore, with probability one, we obtain that
  \begin{equation}
    \liminf_{t \to \infty} \frac1t \log\left( \frac{f(m_t)}{f(m_0)} \right) \geq - C \enspace.
  \end{equation}

  In light of \Cref{prop:f-norm}, with probability one, we obtain that
  \begin{equation}
    \liminf_{t \to \infty} \frac1t \log\left( \frac{ \norm{m_t - x^*} }{ \norm{m_0 - x^*} } \right) \geq - \frac{C}{2} \enspace.
  \end{equation}
  This completes the proof.
\end{proof}

We remark that the order of the lower bound of the lower convergence rate $\Omega_{d \to \infty}\left(\exp\left(-\frac{1}{d}\right)\right)$ derived in \Cref{theorem:upper} matches the upper bound of the upper convergence rate $O_{d \to \infty}\left(\exp\left(-\frac{1}{d}\right)\right)$ derived in \Cref{theorem:lower}.
Therefore, in terms of the search space dimension $d$, we obtain the matching convergence rate bound of $\Theta_{d \to \infty}\left(\exp\left(-\frac{1}{d}\right)\right)$. This is consistent with the implications of the previous work on the sphere function \cite{akimoto2018drift,jagerskupper2007algorithmic}. In contrast, the order of the lower convergence rate bound with respect to the condition number $\Cond(H)$ of the Hessian matrix of the objective function is rather loose compared to that obtained in the previous work  \cite{jagerskupper20061+} on a specific convex quadratic function mentioned in the introduction, which is $\Omega_{d \to \infty}\left(\exp\left(-\frac{1}{d \cdot \Cond(H)}\right)\right)$.

\section{Discussion}\label{sec:discussion}

%
%

    \paragraph{Conclusion}

        In this work, the convergence rate of the (1+1)-ES on the potentially ill-conditioned function, general convex quadratic function, is analyzed.
        It is revealed that the upper convergence rate is in $O_{d\to\infty}\left(\exp\left(-\frac{L}{\Tr(H)}\right)\right)$ and that the lower convergence rate is in $\Omega_{d\to\infty}\left(\exp\left(-\frac{1}{d}\right)\right)$.
        The order of the upper convergence rate in terms of both the dimension $d$ and the Hessian $H$ is derived for the first time.
        Furthermore, our analysis on $\Tr(H)/L$ is superior to that on $d\cdot\Cond(H)$ (partly shown in \cite{jagerskupper20061+}) in that it reveals the impact of the distribution of the eigenvalues of $H$.
        Thus, it theoretically suggests that the ill-conditioned problem is not only in the ratio between the greatest and smallest eigenvalues but also in heaviness of the distribution of the eigenvalues, at least for the (1+1)-ES.
        In addition to the upper convergence rate, we show that the lower convergence rate regarding $d$ on a general convex quadratic function, which is suggested on a portion of convex quadratic function in \cite{akimoto2018drift, jagerskupper2003analysis, jagerskupper20061+, jagerskupper2007algorithmic}.

    \paragraph{Discussion}

        Furthermore, we clarify the limitation of the current analysis.
        For $f \in \mathcal{Q}_{d,\kappa}$, we have $\Tr(H^2) / \Tr(H)^2 \to 0$ for $d \to \infty$. With the definition of $B_H^{\mathrm{low}}$ \eqref{eq:cor:b_low}, condition \eqref{eq:ptarget_condition} for the limit $d \to \infty$ reads
        \begin{equation}
            \frac{4}{\sqrt{2\pi}} >
            2\Phi^{-1}\left(1-p_\mathrm{target}\right)
            \enspace.\label{eq:pcond-discuss}
        \end{equation}
        Hence, $\aup$ and $\adown$ are required to satisfy $p_\mathrm{target} > \Phi^{-1}\left(-\sqrt{\frac{2}{\pi}}\right)$ in the limit $d \to \infty$.
        Note that $\Phi\left(-\sqrt{\frac{2}{\pi}}\right) \approx 0.212$, and then, the classic $1/5$-success rule \cite{kern2004learning, rechenberg1973evolution} is slightly out of the scope of the current study.
        For a finite $d$, the requirement \eqref{eq:ptarget_condition} on $\aup$ and $\adown$ is more restrictive, and it depends on the eigenvalue distribution of $H$. However, as long as $f \in \mathcal{Q}_{d,\kappa}$ and $d \gg \kappa$, the condition can be approximated with \eqref{eq:pcond-discuss}, which implies that one need not tune these hyper-parameters depending on $H$. For $f$ satisfying \eqref{eq:tr_cond} but $d \not\gg \kappa$, our theorem still guarantees the linear convergence but requires $\aup$ and $\adown$ to be carefully tuned depending on $H$ to satisfy \eqref{eq:ptarget_condition}.
          This does not describe the reality: we observe empirically that the (1+1)-ES with $1/5$-success rule converges linearly even for cases with $d < \kappa$.


        The previous work \cite{morinaga2019generalized} takes almost the same approach, in particular, the same potential function as ours.
        The most distinct difference is the separation method of the state space $\Theta$ on the step size $\sigma$ in \Cref{lemma:cases,lemma:too_small_sigma,lemma:too_large_sigma,lemma:reasonable_sigma}.
        In \cite{morinaga2019generalized}, if the setting of the constants is ignored, they defined a reasonable step-size range in the form of $[l\cdot\sigma/\sqrt{f(m_t)}, u\cdot\sigma/\sqrt{f(m_t)}]$, while we defined it in the form of $[l\cdot\sigma/\sqrt{f(m_t)}, u\cdot\sigma/\norm{\nabla f(m_t)}]$.
        This change made it possible to bound the success probability in each scenario and the expected one-step progress in the case of the reasonable step-size more tightly.
        Our approach, analyzing the expected decrease in the potential function \eqref{eq:potential} on each proportion of $\Theta$ separated by $[l\cdot\sigma/\sqrt{f(m_t)}, u\cdot\sigma/\norm{\nabla f(m_t)}]$, does not have such problems and leads to a tighter upper convergence rate bound, as proved in \Cref{sec:upperconvergencerate}.

        The analysis of the (1+1)-ES is also important in terms of demonstrating the potential and limitation of the continuous BBO algorithms.
        In fact, many theoretical studies \cite{golovin2019gradientless, ghadimi2013stochastic, nesterov2017random, balasubramanian2018zeroth} on derivative-free algorithms adopt settings that utilize the properties of the objective function other than the function value for theoretical analysis.
        The current study on the (1+1)-ES exploits only the function value in the optimization and does not bring additional properties of the objective function (such as the Lipschitz constant or the condition number of the Hessian) into the algorithm parameters.
        In other words, the current study is performed in a purely black-box setting.

    \paragraph{Future Work}

            The existence of a constant bound of the upper convergence rate of the (1+1)-ES on the $\alpha$-strongly convex and $\gamma$-Lipschitz smooth function is clearly shown in \cite{morinaga2019generalized}, as mentioned.
            Considering such a function is a superset of convex quadratics, it might be possible to state that the upper convergence rate is $O_{d\to\infty}\left(\exp\left(-\frac{\alpha}{d\cdot \gamma}\right)\right)$ to match our result; however, it is still unclear what is missing for its proof.

            If we bring the CMA mechanism into \Cref{algo}, intuitively, the convergence rate improves on a severely ill-conditioned convex quadratic function, i.e., the case $\Cond(H)$ is considerably larger than $1$.
            However, to date, a promising approach to estimate the convergence rate of the CMA-ES theoretically has hardly been established for any class of the function.
            Exploring the possibility of expanding the applicable range of the analysis scheme is also an important future work in terms of the class of algorithm.

            Regarding the trace or the condition number of $H$, the derived lower convergence rate still does not match the upper convergence rate, which seems to be rigorous, although it matches with respect to $d$ .
            However, previous works \cite{jagerskupper20061+, jagerskupper2003analysis, jagerskupper2007algorithmic} attain the matching order of the convergence rate with respect to the condition number of $H$ on a specific convex quadratic function with an overwhelming probability.
            There is room for consideration as to which method is suitable for estimating the lower convergence rate, although our method seems to have potential in estimating the upper convergence rate.

\begin{acks}

This work is partially supported by JSPS KAKENHI Grant Number 19H04179.

\end{acks}

\balance

\clearpage
\appendix

\section{Proof}

\subsection{Proof of \Cref{prop:invariance}}\label{apdx:prop:invariance}
\begin{proof}[Proof of \Cref{prop:invariance}]

  For the first claim, it is sufficient to show that $\mathcal{G}(\theta, z; f) = \mathcal{G}(\theta, z; g \circ f)$ for all $\theta \in \Theta$ and $z \in \R^d$. This is trivial because $f(m + \sigma \cdot z) \leq f(m) \Leftrightarrow g(f(m + \sigma \cdot z)) \leq g(f(m))$.

  For the second claim, we first show that $\mathcal{G}((m, \sigma), z; f) = \mathcal{G}((m+x^*, \sigma), z; f \circ T)$ for all $\theta \in \Theta$ and $z \in \R^d$. Because $f(T(m + x^* + \sigma \cdot z)) \leq f(T(m + x^*)) \Leftrightarrow f(m + \sigma \cdot z) \leq f(m)$, it is obvious that $\mathcal{G}((m, \sigma), z; f) = \mathcal{G}((m+x^*, \sigma), z; f \circ T)$. Assume that $\tilde{m}_{t} = m_t + x^*$ and $\tilde{\sigma}_t = \sigma_t$. Next, we have $\mathcal{G}(\theta_t, z; f) = \mathcal{G}(\tilde{\theta}_t, z; f \circ T)$, and hence $\tilde{m}_{t+1} = m_{t+1} + x^*$ and $\tilde{\sigma}_{t+1} = \sigma_{t+1}$. Because the assumption holds for $t = 0$, by mathematical induction, we obtain the second claim.
\end{proof}
\hrule

\subsection{Proof of \Cref{prop:convergence_rate}}\label{apdx:prop:convergence_rate}
        \begin{proof}[Proof of \Cref{prop:convergence_rate}]
          Let $Z_{t+1} = X_{t+1} - \E[X_{t+1} \mid \mathcal{F}_t]$.
          Next, $\{Z_{t}\}_{t \geq 1}$ is a martingale difference sequence adapted to $\{\mathcal{F}_{t}\}$.
          Suppose that C3 holds.
          C3 immediately implies that $\sum_{t=1}^{\infty}\E[Z_{t}^2 \mid \mathcal{F}_{t-1}] / t^2 < \infty$. Subsequently, from the strong law of large numbers of martingale \cite{chow1967strong}, we obtain $\lim_{t\to\infty} \frac1t \sum_{i=1}^{t} Z_t = 0$ almost surely. Next, we have
          \begin{align}
            \frac1t (X_t - X_0)
            &= \frac1t \sum_{i=1}^{t}(X_i - X_{i-1})
            \\
            &= \frac1t \sum_{i=1}^{t}(Z_i + \E[X_i - X_{i-1} \mid \mathcal{F}_{i-1}])
            \\
            &= \frac1t \sum_{i=1}^{t}Z_i + \frac1t \sum_{i=1}^{t} \E[X_i - X_{i-1} \mid \mathcal{F}_{i-1}]
              \enspace.
          \end{align}
          We obtain \eqref{eq:prop:upper} by taking $\limsup$ of the equation above and using C1 as well as $\limsup_{t\to\infty} \frac1t \sum_{i=1}^{t} Z_t = 0$.
          Similarly, we obtain \eqref{eq:prop:lower} by taking $\liminf$ and using C2 as well as $\liminf_{t\to\infty} \frac1t \sum_{i=1}^{t} Z_t = 0$. This completes the proof.
        \end{proof}
\hrule

\subsection{Proof of \Cref{prop:f-norm}}\label{apdx:prop:f-norm}

\begin{proof}[Proof of \Cref{prop:f-norm}]
  In light of \Cref{prop:invariance}, we have $\theta_t = S(\tilde{\theta}_t)$ for all $t \geq 0$.
  Next, $\norm{m_t - x^*} = \norm{\tilde{m}_t}$ for all $t \geq 0$.
  Let $L$ and $U$ be the smallest and greatest eigenvalues of $H$, respectively.
  Next, for any $x \in \R^d$, we have
  \begin{equation}
    \frac{2 h(x)}{ U } \leq \norm{x}^2 \leq \frac{2 h(x)}{ L } \enspace.
  \end{equation}
In particular, we obtain
\begin{equation}
  \log \left(\frac{L}{U}\right)
  \leq 2\log\left(\frac{ \norm{\tilde{m}_t}} { \norm{\tilde{m}_0} }\right) - \log\left(\frac{h(\tilde{m}_t)}{h(\tilde{m}_0)}\right)
  \leq \log \left(\frac{U}{L}\right)  \enspace.
\end{equation}
Taking $\limsup$ and $\liminf$ after multiplying all terms by $\frac{1}{t}$, we obtain \Cref{eq:limsup,eq:liminf}.
\end{proof}
\hrule

\subsection{Proof of \Cref{lemma:qualitygain}}\label{apdx:lemma:qualitygain}
\begin{proof}[Proof of \Cref{lemma:qualitygain}]

  For a general convex function $f: \R^d \to \R$, we have
  \begin{enumerate}
  \item $I\{f(m + \sigma z) \leq f(m)\} \leq I\{ \langle \nabla f(m), z \rangle \leq 0\}$, hence $I\{f(m + \sigma z) \leq f(m)\} = I\{ \langle \nabla f(m), z \rangle \leq 0\}I\{f(m + \sigma z) \leq f(m)\}$;
  \item $h(z) = (f(m + \sigma z) - f(m) ) I\{ \langle \nabla f(m), z \rangle \leq 0\}$ and $g(z) = I\{f(m + \sigma z) \leq f(m)\}$ are negatively correlated, i.e., $(h(z_1) - h(z_2))(g(z_1) - g(z_2)) \leq 0$ for all $z_1, z_2 \in \R^d$.
  \end{enumerate}
  Because $h$ and $g$ are negatively correlated, we have $\E[h(z)g(z)] \leq \E[h(z)]\E[g(z)]$.
  Moreover, we have $\E[g(z)] = E[I\{f(m + \sigma z) \leq f(m)\}] = \Pr[f(m + \sigma z) \leq f(m)]$.

  Now, we suppose that $f$ is a convex quadratic function, $\nabla ^2 f(x) = H$.
  Letting $e = \nabla f(m) / \| \nabla f(m) \|$ and $z_e = \langle e, z \rangle$, we have
  \begin{align}
    \E[h(z)]
    &=\E[(f(m + \sigma z) - f(m)), I\{\langle \nabla f(m), z\rangle \leq 0\}],
    \\
    &=
      \E\left[\left(\sigma \langle \nabla f(m), z \rangle + \frac{\sigma^2 }{2} z^\mathrm{T} H z \right) I\{\langle \nabla f(m), z\rangle \leq 0\}\right]
    \\
    &=
      \E\left[\left(\sigma \| \nabla f(m) \| z_e + \frac{\sigma^2 }{2} z^\mathrm{T} H z \right) I\{z_e \leq 0\}\right]
    \\
    &= - \frac{\sigma \| \nabla f(m) \|}{\sqrt{2\pi}} + \frac{\sigma^2 Tr(H)}{4}
    \\
    &= \sigma \| \nabla f(m) \| \left(- \frac{1}{\sqrt{2\pi}} + \frac{1}{4} \frac{\sigma \Tr(H)}{ \| \nabla f(m) \|} \right)
      \enspace.
  \end{align}
  This completes the proof.
\end{proof}
\hrule

\subsection{Proof of \Cref{lemma:variancebound}}\label{apdx:lemma:variancebound}

\begin{proof}[Proof of \Cref{lemma:variancebound}]
  Let $\lambda_i(H)$ be the $i$-th greatest eigenvalue of $H$ for $i = 1, \dots, d$,
  and $\sigma_z^* = \argmin_{\sigma_z\geq 0}f(m + \sigma_z \cdot z)$.
  We have
  \begin{align}
    &\left|\log\left( \frac{f(m + \sigma \cdot z) }{ f(m)} \right) \cdot \ind{f(m + \sigma \cdot z) \leq f(m)} \right|
    \\
    &\leq \left|\log\left( \frac{\min_{\sigma_z \geq 0} f(m + \sigma_z \cdot z)}{f(m)} \right)\right|
    \\
    &= \left|\log\left( 1 + \frac{\sigma_z^* m^\mathrm{T} H z}{f(m)} + \frac{(\sigma_z^*)^2 z^\mathrm{T} H z}{2f(m)} \right)\right|
    \\
    &=
    \left|\log\left(1+\frac{1}{m^\mathrm{T} H m \cdot z^\mathrm{T}Hz}\left((\sigma_z^* \cdot z^\mathrm{T}Hz + m^\mathrm{T}Hz)^2 - (m^\mathrm{T}Hz)^2\right)\right)\right|
    \label{eq:progress_rate_with_12xHx}
    \enspace.
    \end{align}
    Note that we defined $f : x\mapsto \frac12 x^\mathrm{T} H x$.
    Remember that we defined $\sigma_z^*\geq 0$ so that the logarithms inside the absolute value of \eqref{eq:progress_rate_with_12xHx} are minimized, that is,
    \begin{align}
        \sigma_z^* =
        \begin{cases}
            { 0 \enspace \text{ if } m^\mathrm{T}Hz \geq 0 \enspace,}\\
            { -\frac{m^\mathrm{T}Hz}{z^\mathrm{T}Hz} \enspace\text{ if } m^\mathrm{T}Hz \leq 0 \enspace.}
        \end{cases}
    \end{align}
    Therefore, it holds
    \begin{equation}
        (\sigma_z^* \cdot z^\mathrm{T}Hz + m^\mathrm{T}Hz)^2 - (m^\mathrm{T}Hz)^2 = -\min( m^\mathrm{T}Hz, 0)^2
        \enspace.
    \end{equation}
    Then, the RHS of \eqref{eq:progress_rate_with_12xHx} is deformed as
    \begin{align}
    & \left|\log\left( 1 - \frac{\min(m^\mathrm{T} H z, 0)^2}{z^\mathrm{T} H z \cdot m^\mathrm{T} H m }\right)\right|
    \\
    &\leq \left|\log\left( 1 - \frac{(m^\mathrm{T} H z)^2}{z^\mathrm{T} H z \cdot m^\mathrm{T} H m }\right)\right|
    \\
    &= \left| \log\left( 1 - \left\langle \frac{\sqrt{H} m}{\lVert \sqrt{H} m\Vert}, \frac{\sqrt{H} z}{ \lVert \sqrt{H} z\rVert} \right\rangle^2\right)\right| \label{proof:eq:variancebound}
    \enspace.
  \end{align}
  The right-hand side (RHS) is stochastically dominated by
  \begin{align}
    Z
    &:= -  \log\left( 1 - \frac{\lambda_{1}(H) \mathcal{N}_1^2}{\sum_{i=1}^{d} \lambda_{i}(H) \mathcal{N}_i^2}\right),
    \\
    &= \log\left(1 + \frac{\lambda_{1}(H) \mathcal{N}_1^2}{\sum_{i=2}^{d} \lambda_{i}(H) \mathcal{N}_i^2}\right)
      \enspace,
  \end{align}
  i.e., the cumulative density of the RHS of \eqref{proof:eq:variancebound} is upper-bounded by the cumulative density of $Z$, where $\mathcal{N}_1, \dots, \mathcal{N}_d$ are independent and standard normally distributed random variables.
  Next,
  \begin{align}
    \E[\exp(Z)]
    &= 1 + \E\left[\frac{ \lambda_{1}(H) \mathcal{N}_1^2}{\sum_{i=2}^{d} \lambda_{i}(H) \mathcal{N}_i^2}\right].
    \\
    &\leq 1 + \frac{\lambda_{1}(H)}{(d - 1) \lambda_{d}(H)} \E\left[\frac{ \mathcal{N}_1^2}{\frac{1}{d-1} \sum_{i=2}^{d} \mathcal{N}_i^2} \right],
    \\
    &= 1 + \frac{\lambda_{1}(H)}{(d - 3) \lambda_{d}(H)}
      \enspace,
  \end{align}
  where for the last equality, we utilized the fact that $\frac{ \mathcal{N}_1^2}{\frac{1}{d-1} \sum_{i=2}^{d} \mathcal{N}_i^2}$ is F-distributed with degrees of freedom of ($1$, $d-1$) and its expected value is $(d-1)/(d-3)$ if $d > 3$.
\end{proof}
\hrule

\subsection{Proof of \Cref{lemma:successprobability}}\label{apdx:lemma:successprobability}
\begin{proof}[Proof of \Cref{lemma:successprobability}]
  First, by using Chebyshev's inequality and $\E[z^\mathrm{T} H z] = \Tr(H)$, we obtain
  \begin{align}
    \MoveEqLeft[1]
    \Pr\left[\left| \frac{z^\mathrm{T} H z}{\Tr(H)} - 1 \right|\geq \epsilon\right]
    \leq
    \frac{1}{\epsilon^2}
    \cdot \Var\left[\frac{z^\mathrm{T} H z}{\Tr(H)}\right]
    =
    \frac{2}{\epsilon^2} \cdot \frac{\Tr(H^2)}{\Tr(H)^2}
    \enspace.
  \end{align}

  Let $e = \nabla f(m) / \| \nabla f(m) \|$ and $z_e = \langle e, z \rangle$.
  We can write the success probability as
  \begin{align}
    \MoveEqLeft[1]\Pr[ f(m + \sigma z) \leq f(m)]
    \\
    =&
      \Pr\left[\sigma \langle \nabla f(m), z \rangle \leq - \frac{\sigma^2 }{2} z^\mathrm{T} H z \right]
    \\
    =&
      \Pr\left[ \sigma \| \nabla f(m) \| z_e \leq - \frac{\sigma^2 }{2} z^\mathrm{T} H z \right]
    \\
    =&  \Pr\left[ z_e \leq - \frac{\sigma \Tr(H)}{2 \| \nabla f(m) \| } \frac{z^\mathrm{T} H z}{\Tr(H)} \right].
  \end{align}

  Next, we derive the upper bound.
  \begin{align}
    \MoveEqLeft[1]\Pr\left[f(m + \sigma z) \leq f(m) \right]
    \\
    =&
       \Pr\left[ z_e \leq - \frac12 \frac{\sigma \Tr(H)}{\norm{\nabla f(m)}} \frac{z^\mathrm{T} H z}{\Tr(H)} \right]
    \\
    =&
       \Pr\left[ \left(z_e \leq - \frac12 \frac{\sigma \Tr(H)}{\norm{\nabla f(m)}}\cdot\frac{z^\mathrm{T} H z}{\Tr(H)}\right)\cap\left(\frac{z^\mathrm{T} H z}{\Tr(H)}\geq 1-\epsilon\right)\right]
    \\
    &+\Pr\left[ \left(z_e \leq - \frac12 \frac{\sigma \Tr(H)}{\norm{\nabla f(m)}}\cdot\frac{z^\mathrm{T} H z}{\Tr(H)}\right)\cap\left(\frac{z^\mathrm{T} H z}{\Tr(H)}< 1-\epsilon\right)\right].
    \\
    <&
          \Pr\left[ z_e \leq - \frac12 \frac{\sigma \Tr(H)}{\norm{\nabla f(m)}}\cdot(1-\epsilon)\right]
          +\Pr\left[ \frac{z^\mathrm{T} H z}{\Tr(H)}\leq 1-\epsilon\right]
    \\
    \leq&
          \Phi\left( - \frac12 \frac{\sigma \Tr(H)}{\norm{\nabla f(m)}}\cdot(1-\epsilon)\right)
          + \frac{2}{\epsilon^2} \cdot \frac{\Tr(H^2)}{\Tr(H)^2}
          \enspace.
  \end{align}

  Finally, we derive the lower bound.
  \begin{align}
    \MoveEqLeft[1]\Pr\left[f(m + \sigma z) > f(m) \right]
    \\
    =&
       \Pr\left[ z_e > - \frac12 \frac{\sigma \Tr(H)}{\norm{\nabla f(m)}} \frac{z^\mathrm{T} H z}{\Tr(H)} \right]
    \\
    =&
       \Pr\left[ \left(z_e > - \frac12 \frac{\sigma \Tr(H)}{\norm{\nabla f(m)}}\cdot\frac{z^\mathrm{T} H z}{\Tr(H)}\right)\cap\left(\frac{z^\mathrm{T} H z}{\Tr(H)}\leq 1+\epsilon\right)\right]
    \\
    &+\Pr\left[ \left(z_e > - \frac12 \frac{\sigma \Tr(H)}{\norm{\nabla f(m)}}\cdot\frac{z^\mathrm{T} H z}{\Tr(H)}\right)\cap\left(\frac{z^\mathrm{T} H z}{\Tr(H)}> 1+\epsilon\right)\right].
    \\
    <&
          \Pr\left[ z_e > - \frac12 \frac{\sigma \Tr(H)}{\norm{\nabla f(m)}}\cdot(1+\epsilon)\right]
          +\Pr\left[ \frac{z^\mathrm{T} H z}{\Tr(H)}>1+\epsilon\right]
    \\
    \leq&
          1 - \Phi\left( - \frac12 \frac{\sigma \Tr(H)}{\norm{\nabla f(m)}}\cdot(1+\epsilon)\right)
          + \frac{2}{\epsilon^2} \cdot \frac{\Tr(H^2)}{\Tr(H)}
          \enspace.
  \end{align}
  Hence, we have
    \begin{multline}
    \Pr\left[f(m + \sigma z) \leq f(m) \right] \\
    > \Phi\left( - \frac12 \frac{\sigma \Tr(H)}{\norm{\nabla f(m)}}\cdot(1+\epsilon)\right)
          - \frac{2}{\epsilon^2} \cdot \frac{\Tr(H^2)}{\Tr(H)^2}
          \enspace.
  \end{multline}
  This completes the proof.
\end{proof}
\hrule

\subsection{Proof of \Cref{cor:successprobability}}\label{apdx:cor:successprobability}
\begin{proof}[Proof of \Cref{cor:successprobability}]
  First, we prove \cref{eq:cor:successprobability:large} in our paper.
  In light of \Cref{lemma:successprobability}, for each $q \in \left(0, \frac12\right)$, to show that there exists an $\epsilon > 0$ such that
    \begin{equation}
        q\leq \Phi\left(-\frac12 \frac{\sigma \Tr(H)}{\norm{\nabla f(m)}}\cdot (1+\epsilon)\right) - \frac{2}{\epsilon^2}\cdot\frac{\Tr(H^2)}{\Tr(H)^2}
        \enspace,
        \label{eq:4}
    \end{equation}
  is sufficient to prove
\begin{equation}
    \frac{\sigma \Tr(H)}{\norm{\nabla f(m)}}\leq B_H^\mathrm{high}(q)
    \Rightarrow \mathrm{Pr}[f(m+\sigma z)\leq f(m)]>q
    \label{eq:5}
    \enspace.
\end{equation}
  Therefore, our following discussion is under $(m, \sigma)$ which satisfies $\frac{\sigma \Tr(H)}{\norm{\nabla f(m)}}\leq B_H^\mathrm{high}(q)$.
In light of Lemma 6, for any $m\in\R^d, \sigma>0$, and $\epsilon>0$, it holds
\begin{equation}
    \Phi\left(-\frac12 \frac{\sigma \Tr(H)}{\norm{\nabla f(m)}}\cdot (1+\epsilon)\right) - \frac{2}{\epsilon^2}\cdot\frac{\Tr(H^2)}{\Tr(H)^2}
    <
    \mathrm{Pr}[f(m+\sigma z)\leq f(m)]
    \enspace.
    \label{eq:6}
\end{equation}
Hence, given any $q\in(0, 1/2)$, the RHS of \cref{eq:5} holds if the following condition is satisfied :
\begin{multline}
    q\leq \Phi\left(-\frac12 \frac{\sigma \Tr(H)}{\norm{\nabla f(m)}}\cdot (1+\epsilon)\right) - \frac{2}{\epsilon^2}\cdot\frac{\Tr(H^2)}{\Tr(H)^2}
    \\
    \Leftrightarrow
    \frac{\sigma \Tr(H)}{\norm{\nabla f(m)}} \leq \frac{ 2 \Phi^{-1} \left(1 - \left(q + \frac{2}{\epsilon^2} \cdot \frac{\Tr(H^2)}{\Tr(H)^2}\right) \right)}{ 1 + \epsilon}
    =: B_H^\mathrm{high}(q; \epsilon)
    \label{eq:7}
    \enspace.
\end{multline}
Note that $\frac{\sigma \Tr(H)}{\norm{\nabla f(m)}}$ can never be negative with the definitions of $\sigma, H$.
For any $q\in(0, 1/2)$, there exists a positive constant $\epsilon$ which satisfies $\epsilon>\sqrt{\frac{4}{1-2q}\cdot\frac{\Tr(H^2)}{\Tr(H)^2}}$,
and if $\epsilon>\sqrt{\frac{4}{1-2q}\cdot\frac{\Tr(H^2)}{\Tr(H)^2}}$, it holds
$B_H^\mathrm{high}(q; \epsilon)>0$.
Therefore, now we see that there exists $\epsilon_H^\text{high}(q) > \sqrt{ \frac{ 4 }{1 - 2q} \cdot \frac{\Tr(H^2)}{\Tr(H)^2} }$ such that $B_H^\text{high}(q) = B_H^\text{high}(q; \epsilon^\text{high}(q))$.
Finally, as \cref{eq:6} holds for any $\epsilon>0$, given any $q\in(0, 1/2)$, the RHS of \cref{eq:5} also holds if
\begin{equation}
    \frac{\sigma \Tr(H)}{\norm{\nabla f(m)}} \leq \sup_{\epsilon>\sqrt{\frac{4}{1-2q}\cdot\frac{\Tr(H^2)}{\Tr(H)^2}}}\frac{ 2 \Phi^{-1} \left(1 - \left(q + \frac{2}{\epsilon^2} \cdot \frac{\Tr(H^2)}{\Tr(H)^2}\right) \right)}{ 1 + \epsilon}
    =: B_H^\mathrm{high}(q)
    \enspace.
\end{equation}
This completes the proof of \cref{eq:5} and then we obtain \cref{eq:cor:successprobability:large} in our paper.

  Next, we prove the properties of $B_H^\text{high}$.
  First, note that $B_H^\text{high}(q; \epsilon)$ is continuous and strictly decreasing with respect to $q$ for each $\epsilon$.
  We prove the right-continuity and strict-decrease of $B_H^\text{high}$ by contradiction.
  If $B_H^\text{high}$ does not strictly decrease at a point $q$ in the domain, then there must exist $q' > q$ such that $B_H^\text{high}(q) \leq B_H^\text{high}(q')$.
  However, $B_H^\text{high}(q') = B_H^\text{high}(q'; \epsilon^\text{high}(q')) < B_H^\text{high}(q; \epsilon^\text{high}(q')) \leq B_H^\text{high}(q; \epsilon^\text{high}(q)’) = B_H^\text{high}(q’)$, which contradicts $B_H^\text{high}(q) \leq B_H^\text{high}(q')$.
  Hence, $B_H^\text{high}$ is strictly decreasing.
  If $B_H^\text{high}$ is not right-continuous at a point $q$, there must exist $\delta > 0$ such that $B_H^\text{high}(q) - \delta \geq B_H^\text{high}(q')$ for all $q' > q$.
  Because $B_H^\text{high}(q; \epsilon^\text{high}(q)) = B_H^\text{high}(q’) > B_H^\text{high}(q') = B_H^\text{high}(q'; \epsilon^\text{high}(q')) \geq B_H^\text{high}(q'; \epsilon^\text{high}(q))$, we have $\abs{B_H^\text{high}(q) - B_H^\text{high}(q')} \leq \abs{B_H^\text{high}(q; \epsilon^\text{high}(q)) - B_H^\text{high}(q'; \epsilon^\text{high}(q))}$. The LHS must be no smaller than $\delta$ for any $q' > q$, whereas from the continuity of $B_H^\text{high}(\cdot; \epsilon^\text{high}(q))$, the RHS satisfies $\lim_{q' \downarrow q} \abs{B_H^\text{high}(q; \epsilon^\text{high}(q)) - B_H^\text{high}(q'; \epsilon^\text{high}(q))} = 0$, which is a contradiction.
  Hence, $B_H^\text{high}$ is right-continuous.
  It is trivial to see that $B_{H}^\mathrm{high}(q) \leq 2 \Phi^{-1}(1-q)$ for all $q \in \left(0, \frac12\right)$.

  Similarly, we prove \cref{eq:cor:successprobability:small}.
  In light of \Cref{lemma:successprobability}, for each $q \in \left( 2 \cdot \frac{\Tr(H^2)}{\Tr(H)^2} , \frac{1}{2}\right)$, it is sufficient to show that there exists an $\epsilon \in (0, 1)$ such that the RHS of \eqref{eq:lemma:successprobability} is no greater than $q$.
  The condition $\epsilon < 1$ is necessary to have the RHS of \eqref{eq:lemma:successprobability} smaller than $1/2$.
  By solving this inequality for $\frac{\sigma \Tr(H)}{\norm{\nabla f(m)}}$, we obtain
  \begin{equation}
    \frac{\sigma \Tr(H)}{\norm{\nabla f(m)}} \geq \frac{ 2 \Phi^{-1} \left(1 - \left(q - \frac{2}{\epsilon^2} \cdot \frac{\Tr(H^2)}{\Tr(H)^2}\right) \right)}{ 1 - \epsilon}  =: B_H^\text{low}(q; \epsilon)\enspace.
    \label{eq:cor:sp:small:eps}
  \end{equation}
  That is, if there exists an $\epsilon \in (0, 1)$ such that the condition above holds, then we have $\Pr\left[f(m + \sigma z) \leq f(m) \right] < q$.
  In contrast, for each $q \in \left( 2 \cdot \frac{\Tr(H^2)}{\Tr(H)^2} , \frac{1}{2}\right)$, we can easily see that there exists $\epsilon_H^\text{low}(q) \in \left(\sqrt{ \frac{ 2 }{q} \cdot \frac{\Tr(H^2)}{\Tr(H)^2} }, 1\right)$ such that $B_H^\text{low}(q) = B_H^\text{low}(q; \epsilon^\text{low}(q))$.
  Hence, we obtain \cref{eq:cor:successprobability:large}.

  Finally, we prove the properties of $B_H^\text{low}$.
  Note that $B_H^\text{high}(q; \epsilon)$ is continuous and strictly decreasing with respect to $q$ for each $\epsilon$.
  We prove the left-continuity and strict decrease of $B_H^\text{low}$ by contradiction.
  If $B_H^\text{low}$ is not strictly decreasing at a point $q$ in the domain, then there must exist $q' < q$ such that $B_H^\text{low}(q) \geq B_H^\text{low}(q')$.
  However, $B_H^\text{low}(q') = B_H^\text{low}(q'; \epsilon^\text{low}(q')) > B_H^\text{low}(q; \epsilon^\text{low}(q')) \geq B_H^\text{low}(q; \epsilon^\text{low}(q)’) = B_H^\text{low}(q’)$, which contradicts $B_H^\text{low}(q) \geq B_H^\text{low}(q')$.
  Hence, $B_H^\text{high}$ is strictly decreasing.
  If $B_H^\text{low}$ is not left-continuous at a point $q$, there must exist $\delta > 0$ such that $B_H^\text{low}(q) + \delta \leq B_H^\text{low}(q')$ for all $q' < q$.
  Because $B_H^\text{low}(q; \epsilon^\text{low}(q)) = B_H^\text{low}(q’) < B_H^\text{low}(q') = B_H^\text{low}(q'; \epsilon^\text{low}(q')) \leq B_H^\text{low}(q'; \epsilon^\text{low}(q))$, we have $\abs{B_H^\text{low}(q) - B_H^\text{low}(q')} \leq \abs{B_H^\text{low}(q; \epsilon^\text{low}(q)) - B_H^\text{low}(q'; \epsilon^\text{low}(q))}$. The LHS must be no smaller than $\delta$ for any $q' < q$, whereas from the continuity of $B_H^\text{low}(\cdot; \epsilon^\text{low}(q))$, the RHS satisfies $\lim_{q' \uparrow q} \abs{B_H^\text{low}(q; \epsilon^\text{low}(q)) - B_H^\text{low}(q'; \epsilon^\text{low}(q))} = 0$, which is a contradiction.
  Hence, $B_H^\text{low}$ is left-continuous.
  It is trivial to see that $B_{H}^\mathrm{high}(q) \geq 2 \Phi^{-1}(1-q)$ for all $q \in \left(0, \frac12\right)$.
\end{proof}
\hrule

\subsection{Proof of \Cref{lemma:cases}}\label{apdx:lemma:cases}

\begin{proof}[Proof of \Cref{lemma:cases}]
  First, we show the existence of a pair of $q^\text{low}$ and $q^\text{high}$. Condition \eqref{eq:tr_cond} implies that $\Tr(H^2) / \Tr(H)^2 < 1/16$. Next, we have $\lim_{q \to 1/2} \sqrt{ \frac{ 2 }{q} \cdot \frac{\Tr(H^2)}{\Tr(H)^2} } < 1/2$. Subsequently, we have $\lim_{q \to 1/2} B_H^\text{low}(q) \leq \lim_{q \to 1/2} B_H^\text{low}(q; \epsilon=1/2)$, and the RHS is
  \begin{equation}
   \lim_{q \to 1/2} B_H^\text{low}(q; \epsilon=1/2) = 4 \Phi^{-1} \left(\frac12 + 8 \cdot \frac{\Tr(H^2)}{\Tr(H)^2}\right)
    < \frac{4}{\sqrt{2\pi}} \enspace.
  \end{equation}
  Hence, $\lim_{q \to 1/2} B_H^\text{low}(q) < \frac{4}{\sqrt{2\pi}}$. Because $B_H^\text{low}$ is left-continuous and strictly decreasing, one can choose $q^\text{low} \in (2\cdot \Tr(H^2) / \Tr(H)^2, 1/2)$ satisfying \eqref{eq:bhqlow}.

  Next, we prove $Q_H > 0$.
  Because of \eqref{eq:bhqlow}, it is sufficient to show that there exists $Q > 0$ such that $B_H^\text{high}(Q) > \frac{4}{\sqrt{2\pi}}$.
  Under condition \eqref{eq:tr_cond}, we have $\lim_{q \to 0} \sqrt{ \frac{ 4 }{1 - 2q} \cdot \frac{\Tr(H^2)}{\Tr(H)^2} } < \frac12$. Therefore, $\lim_{q \to 0}B_H^\text{high}(q) \geq \lim_{q \to 0}B_H^\text{high}(q; \epsilon=1/2)$, and the RHS is
  \begin{equation}
   \lim_{q \to 0} B_H^\text{high}(q; \epsilon=1/2) = \frac43 \Phi^{-1} \left(1 - 8 \cdot \frac{\Tr(H^2)}{\Tr(H)^2}\right)
    > \frac{4}{\sqrt{2\pi}} \enspace.
  \end{equation}
  Hence, $\lim_{q \to 0} B_H^\text{high}(q) \geq > \frac{4}{\sqrt{2\pi}}$. Because $B_H^\text{high}$ is right-continuous and strictly decreasing, there exists $Q > 0$ such that $B_H^\text{high}(q) > \frac{4}{\sqrt{2\pi}}$.
  Hence, $Q_H \geq Q > 0$.

  If $\sigma < B_H^\mathrm{high}(q^\mathrm{high}) \cdot \sqrt{ 2 L f(m) } / \Tr(H)$, using the relation $\sqrt{f(m)} \leq \norm{\nabla f(m)} / \sqrt{2 L}$, we have
  \begin{align*}
    \frac{ \sigma \Tr(H) }{ \norm{\nabla f(m)} } < B_H^\mathrm{high}(q^\mathrm{high}) \enspace.
  \end{align*}
  From \Cref{cor:successprobability}, we find \eqref{eq:lemma:toosmall}.

  If $\sigma > \cdot B_H^\mathrm{low}(q^\mathrm{low}) \cdot \norm{\nabla f(m)} / \Tr(H)$,
  we find \eqref{eq:lemma:toolarge} immediately from \Cref{cor:successprobability}.


  If $\sigma \leq \cdot B_H^\mathrm{low}(q^\mathrm{low}) \cdot \norm{\nabla f(m)} / \Tr(H)$,
  we have
  \begin{align*}
    \frac{ \sigma \Tr(H) }{ \norm{\nabla f(m)} } \leq B_H^\mathrm{low}(q^\mathrm{low}) \leq B_H^\mathrm{high}(Q^\mathrm{low} - \xi^\text{low})
  \end{align*}
  for any $\xi^\text{high} > 0$.
  From \Cref{cor:successprobability}, we find
  \begin{equation}
    \Pr\left[f(m + \sigma z) \leq f(m) \right] \geq Q^\mathrm{low} + \xi^\text{low} \enspace.
  \end{equation}
  Taking $\inf$ over $\xi^\text{low}$, we obtain the LHS of \eqref{eq:lemma:reasonable}.

  Under the condition $B_H^\mathrm{high}(q^\mathrm{high}) \cdot \sqrt{ 2 L f(m) } / \Tr(H) \leq \sigma$, we have
  \begin{align*}
    \frac{ \sigma \norm{\nabla f(m)} }{ f(m) } \geq B_H^\mathrm{high}(q^\mathrm{high}) \frac{ \sqrt{ 2 L } }{ \Tr(H) } \frac{ \norm{\nabla f(m)} }{ \sqrt{ f(m) } }
    \enspace.
  \end{align*}
  Using the relation $\norm{\nabla f(m)} / \sqrt{2 U} \leq \sqrt{f(m)} \leq \norm{\nabla f(m)} / \sqrt{2 L}$ on a convex quadratic function, we obtain
  \begin{align*}
    \frac{ \sigma \norm{\nabla f(m)} }{ f(m) } \geq \frac{ 2 L \cdot B_H^\mathrm{high}(q^\mathrm{high}) }{ \Tr(H) }
    \enspace.
  \end{align*}
  Under the condition $\sigma \leq \cdot B_H^\mathrm{low}(q^\mathrm{low}) \cdot \norm{\nabla f(m)} / \Tr(H)$, we have
  \begin{align*}
    \frac{ \sigma \Tr(H) }{ \norm{\nabla f(m)} } \leq B_H^\mathrm{low}(q^\mathrm{low})\enspace.
  \end{align*}
  From \Cref{lemma:qualitygain}, we obtain \eqref{eq:lemma:reasonable:qualitygain}.
  The negativity follows $B_H^\mathrm{low}(q^\mathrm{low}) < \frac{4}{\sqrt{2\pi}}$.
\end{proof}
\hrule

\subsection{Proof of \Cref{lemma:too_small_sigma}}\label{apdx:lemma:too_small_sigma}
In the following proofs, we use abbreviations of the indicator functions as follows:
\begin{itemize}
\item $\indup = \ind{f(m_t + \sigma_t z_t) \leq f(m_t)}$
\item $\inddown = \ind{f(m_t + \sigma_t z_t) > f(m_t)}$
\item $\inds = \ind{b_s\sqrt{ Lf(m_{t+1})} \geq  \Tr(H)\cdot \sigma_{t+1}}$
\item $\indl = \ind{b_\ell \sqrt{ L f(m_{t+1})} \leq \Tr(H) \cdot \sigma_{t+1}}$
\end{itemize}
Note that $\inds$ and $\indl$ are exclusive to each other, as $b_s<b_\ell$.

\begin{proof}[Proof of \Cref{lemma:too_small_sigma}]

  If $\sigma_t < \frac{b_s\sqrt{L}}{\aup\Tr(H)} \sqrt{f(m_t)}$,
  we have $\frac{b_s\sqrt{Lf(m_t)}}{\Tr(H) \sigma_t}>\aup>1$
  and $\frac{\Tr(H)\sigma_t}{b_\ell\sqrt{Lf(m_t)}}<\frac{b_s}{b_\ell\aup}<1$.
  The potential function at $t$ is
  \begin{align}
    V(\theta_t)
    = \log\left(f(m_t)\right)
    + v \cdot \log \left(\frac{b_s\sqrt{Lf(m_t)}}{\Tr(H) \sigma_t}\right)
    \enspace,
  \end{align}
  and a one-step difference of $V(\theta)$ under the current condition is written as follows:
  \begin{align}
    V(\theta_{t+1}) - V(\theta_t)
    =& \log\left(\frac{f(m_{t+1})}{f(m_t)}\right) \\
    &+ v\cdot\log\left(\frac{b_s\sqrt{ Lf(m_{t+1})}}{\Tr(H) \sigma_{t+1}}\right)\cdot\inds\\
    &+ v\cdot\log\left(\frac{\Tr(H)\sigma_{t+1}}{b_\ell\sqrt{ Lf(m_{t+1})}}\right)\cdot\indl\\
    &- v \cdot \log \left(\frac{b_s\sqrt{Lf(m_t)}}{\Tr(H) \sigma_t}\right)
      \enspace.
  \end{align}
  The RHS is deformed by extracting $\log(f(m_{t+1})/f(m_t))$ and by rewriting $\sigma_{t+1}$ into $\aup\sigma_t$ and $\adown\sigma_t$ using the conditions of $\indup$ and $\inddown$:
  \begin{align}
    V(\theta_{t+1}) - V(\theta_t)
    =& \left(1-(\inds-\indl)\frac{v}{2}\right)\cdot\log\left(\frac{f(m_{t+1})}{f(m_t)}\right) \\
    &+ v\cdot\log\left(\frac{b_s\sqrt{ Lf(m_{t})}}{\Tr(H) \aup\sigma_{t}}\right)\cdot\inds\indup\\
    &+ v\cdot\log\left(\frac{b_s\sqrt{ Lf(m_{t})}}{\Tr(H) \adown\sigma_{t}}\right)\cdot\inds\inddown\\
    &+ v\cdot\log\left(\frac{\Tr(H)\aup\sigma_{t}}{b_\ell\sqrt{ Lf(m_{t})}}\right)\cdot\indl\indup\label{eq:small_indl_indup}\\
    &+ v\cdot\log\left(\frac{\Tr(H)\adown\sigma_{t}}{b_\ell\sqrt{ Lf(m_{t})}}\right)\cdot\indl\inddown\label{eq:small_indl_inddown}\\
    &- v \cdot \log \left(\frac{b_s\sqrt{Lf(m_t)}}{\Tr(H) \sigma_t}\right)
      \enspace.
  \end{align}
  The first term of the RHS of the expression above is non-positive, as $0<v<1$.
  Under the current condition $\sigma_t < \frac{b_s\sqrt{L}}{\aup\Tr(H)} \sqrt{f(m_t)}$,
  \eqref{eq:small_indl_indup} and \eqref{eq:small_indl_inddown} are both non-positive as $b_s<b_\ell$.
  Next, we obtain the upper bound of $V(\theta_{t+1}) - V(\theta_t)$:
  \begin{align}
    V(\theta_{t+1}) - V(\theta_t)
    \leq&
          v\cdot\log\left(\frac{b_s\sqrt{ Lf(m_{t})}}{\Tr(H) \aup\sigma_{t}}\right)\cdot\inds\indup\\
    &+ v\cdot\log\left(\frac{b_s\sqrt{ Lf(m_{t})}}{\Tr(H) \adown\sigma_{t}}\right)\cdot\inds\inddown\\
    &- v \cdot \log \left(\frac{b_s\sqrt{Lf(m_t)}}{\Tr(H) \sigma_t}\right)
    \\
    =&
       (\inds - 1)\cdot v\cdot\log\left(\frac{b_s\sqrt{ Lf(m_{t})}}{\Tr(H) \sigma_{t}}\right)\\
    &- v \cdot \log \left(\frac{b_s\sqrt{Lf(m_t)}}{\Tr(H) \sigma_t}\right)\\
    &+v\cdot\log\left(\frac{\aup}{\adown}\right)\cdot \inds\inddown
      -v\log(\aup)\cdot \inds
    \\
    \leq&
       (\inds - 1)\cdot v\cdot\log\left(\frac{b_s\sqrt{ Lf(m_{t})}}{\Tr(H) \sigma_{t}}\right)\\
    &+v\cdot\log\left(\frac{\aup}{\adown}\right)\cdot \inds\inddown
      -v\log(\aup)
      \enspace.
  \end{align}
  To the last bound of the expression above, we use$\frac{b_s\sqrt{Lf(m_t)}}{\Tr(H) \sigma_t}>\aup$.
  The first term of the RHS of the expression above is non-positive under the current condition $\sigma_t < \frac{\aup b_s\sqrt{L}}{\Tr(H)} \sqrt{f(m_t)}$.
  Moreover, replacing $\inds$ with 1 in the second term provides the following upper bound :
  \begin{align}
    V(\theta_{t+1}) - V(\theta_t)
    \leq
    v\cdot\log\left(\frac{\aup}{\adown}\right)\cdot\inddown.
    -v\log(\aup)
    \enspace.
  \end{align}
  Note that $\E[\inddown] = 1 - \mathrm{Pr}\left[f(m_t + \sigma_t z_t)\leq f(m_t) \mid \mathcal{F}_t\right]$, which is upper-bounded by $1 - q^\mathrm{high}$ in light of \Cref{lemma:cases}.
  By taking the expectation on both sides of the inequality above, we obtain
  \begin{align}
    \MoveEqLeft[2] \E[ V(\theta_{t+1}) - V(\theta_t) \mid \mathcal{F}_t ]
    \\
    &\leq
      v\cdot \log\left(\frac{\aup}{\adown}\right)\cdot \left( 1 - q^\text{high} - \frac{\log(\aup)}{\log(\aup / \adown)} \right).
    \\
    &=
      v\cdot \log\left(\frac{\aup}{\adown}\right)\cdot \left( p_\mathrm{target} - q^\text{high}\right)
      \enspace.
  \end{align}
  Because $v\cdot \log(\aup/\adown) = \min\{ w/4 , \log(\aup / \adown)\}$, we obtain \eqref{eq:lemma:too_small_sigma}.
  This completes the proof.
\end{proof}
\hrule

\subsection{Proof of \Cref{lemma:too_large_sigma}}\label{apdx:lemma:too_large_sigma}
\begin{proof}[Proof of \Cref{lemma:too_large_sigma}]
  If $\sigma_t > \frac{b_\ell}{\sqrt{2}\adown\Tr(H)}\norm{\nabla f(m_t)}$, $\indl = 1$ and naturally $\inds = 0$ because
  \begin{align}
    \frac{b_\ell \sqrt{ L f(m_{t+1})}}{\Tr(H) \sigma_{t+1}}
    \leq
    \frac{b_\ell \sqrt{ L f(m_{t})}}{\Tr(H) \adown\sigma_{t}}
    \leq
    \frac{b_\ell \norm{\nabla f(m_{t})}}{\sqrt{2}\Tr(H) \adown\sigma_{t}}
    <1
    \enspace.
  \end{align}
  As $ \sigma_t > \frac{b_\ell}{\sqrt{2}\adown\Tr(H)}\norm{\nabla f(m_t)} \Rightarrow \frac{b_\ell \sqrt{ L f(m_{t})}}{\Tr(H) \sigma_{t}}< 1$,
  the potential function at $t$ under the current condition is
  \begin{align}
    V(\theta_t)
    =& \log f(m_t)
       + v \cdot \log \left(\frac{\Tr(H)\sigma_t}{b_\ell\sqrt{Lf(m_t)}}\right) \enspace.
  \end{align}
  Next, a one-step difference of the potential function under $ \sigma_t > \frac{b_\ell}{\sqrt{2}\adown\Tr(H)}\norm{\nabla f(m_t)}$ is
  \begin{align}
    V(\theta_{t+1}) - V(\theta_t)
    =& \left(1-(\inds-\indl)\frac{v}{2}\right)\cdot\log\left(\frac{f(m_{t+1})}{f(m_t)}\right) \\
    &+ v\cdot\log\left(\frac{b_s\sqrt{ Lf(m_{t})}}{\Tr(H) \aup\sigma_{t}}\right)\cdot\inds\indup\\
    &+ v\cdot\log\left(\frac{b_s\sqrt{ Lf(m_{t})}}{\Tr(H) \adown\sigma_{t}}\right)\cdot\inds\inddown\\
    &+ v\cdot\log\left(\frac{\Tr(H)\aup\sigma_{t}}{b_\ell\sqrt{ Lf(m_{t})}}\right)\cdot\indl\indup\\
    &+ v\cdot\log\left(\frac{\Tr(H)\adown\sigma_{t}}{b_\ell\sqrt{ Lf(m_{t})}}\right)\cdot\indl\inddown\\
    &- v \cdot \log^+ \left(\frac{b_s\sqrt{Lf(m_t)}}{\Tr(H) \sigma_t}\right)\\
    &-  v \cdot \log^+ \left(\frac{\Tr(H)\sigma_t}{b_\ell\sqrt{Lf(m_t)}}\right)
    \\
    =& \left(1-(\inds-\indl)\frac{v}{2}\right)\cdot\log\left(\frac{f(m_{t+1})}{f(m_t)}\right) \\
    &+v\cdot\log\left(\frac{\aup}{\adown}\right)\cdot\indup
      +v\cdot\log\left(\adown\right)
    \\
    \leq&
          v\cdot\log\left(\frac{\aup}{\adown}\right)\cdot\indup.
          +v\cdot\log\left(\adown\right)
          \enspace.
  \end{align}

  Note that $\E[\indup] = \mathrm{Pr}\left[f(m_t + \sigma_t z_t)\leq f(m_t) \mid \mathcal{F}_t\right]$, which is upper-bounded by $q^\mathrm{low}$ in light of \Cref{lemma:cases}.
  By taking the expectation on both sides of the inequality above, we obtain
  \begin{align}
    \MoveEqLeft[2] \E[ V(\theta_{t+1}) - V(\theta_t) \mid \mathcal{F}_t ]
    \\
    &\leq v\cdot\log\left(\frac{\aup}{\adown}\right)\cdot \left( q^\mathrm{low} - \frac{\log(1/\adown)}{\log(\aup/\adown)}\right)
    \\
    &= v\cdot\log\left(\frac{\aup}{\adown}\right)\cdot \left( q^\mathrm{low} - p_\mathrm{target}\right)
          \enspace.
  \end{align}
  Because $v\cdot \log(\aup/\adown) = \min\{ w/4 , \log(\aup / \adown)\}$, we obtain \eqref{eq:lemma:too_large_sigma}.
  This completes the proof.
\end{proof}
\hrule

\subsection{Proof of \Cref{lemma:reasonable_sigma}}\label{apdx:lemma:reasonable_sigma}
\begin{proof}[Proof of \Cref{lemma:reasonable_sigma}]
  Extracting $\log(f(m_{t+1})/f(m_t))$, $\log(\aup)$ and $\log(\adown)$ from $V(\theta_{t+1})$,
  \begin{align}
    V(\theta_{t+1})
    =& \log\left(f(m_{t+1})\right)-(\inds-\indl)\frac{v}{2}\cdot\log\left(\frac{f(m_{t+1})}{f(m_t)}\right) \\
    &+ v\cdot\log\left(\frac{b_s\sqrt{ Lf(m_{t})}}{\Tr(H) \aup\sigma_{t}}\right)\cdot\inds\indup\\
    &+ v\cdot\log\left(\frac{b_s\sqrt{ Lf(m_{t})}}{\Tr(H) \adown\sigma_{t}}\right)\cdot\inds\inddown\\
    &+ v\cdot\log\left(\frac{\Tr(H)\aup\sigma_{t}}{b_\ell\sqrt{ Lf(m_{t})}}\right)\cdot\indl\indup\\
    &+ v\cdot\log\left(\frac{\Tr(H)\adown\sigma_{t}}{b_\ell\sqrt{ Lf(m_{t})}}\right)\cdot\indl\inddown\\
    \\
    =& \log\left(f(m_{t+1})\right)-(\inds-\indl)\frac{v}{2}\cdot\log\left(\frac{f(m_{t+1})}{f(m_t)}\right) \\
    &+v\cdot\log\left(\frac{\aup}{\adown}\right)\cdot\inds\inddown
      -v\cdot\log\left(\aup\right)\cdot\inds, \\
    &+v\cdot\log\left(\frac{\aup}{\adown}\right)\cdot\indl\indup
      +v\cdot\log\left(\adown\right)\cdot\indl\\
    &+ v\cdot\log\left(\frac{b_s\sqrt{ Lf(m_{t})}}{\Tr(H) \sigma_{t}}\right)\cdot\inds
    \\
    &+ v\cdot\log\left(\frac{\Tr(H)\sigma_{t}}{b_\ell\sqrt{ Lf(m_{t})}}\right)\cdot\indl
      \enspace.
  \end{align}
  Next,
  \begin{align}
    V(\theta_{t+1}) - V(\theta_t)
    =&
    \label{eq:main_terms_of_V_diff_start}
    \left(1-(\inds-\indl)\frac{v}{2}\right)\cdot\log\left(\frac{f(m_{t+1})}{f(m_t)}\right) \\
    &+v\cdot\log\left(\frac{\aup}{\adown}\right)\cdot\inds\inddown
      -v\cdot\log\left(\aup\right)\cdot\inds, \\
    \label{eq:main_terms_of_V_diff_end}
    &+v\cdot\log\left(\frac{\aup}{\adown}\right)\cdot\indl\indup
      +v\cdot\log\left(\adown\right)\cdot\indl\\
    &+ v\cdot\log\left(\frac{b_s\sqrt{ Lf(m_{t})}}{\Tr(H) \sigma_{t}}\right)\cdot\inds
      \label{eq:penalty_start}
    \\
    &+ v\cdot\log\left(\frac{\Tr(H)\sigma_{t}}{b_\ell\sqrt{ Lf(m_{t})}}\right)\cdot\indl\\
    &- v \cdot \log^+ \left(\frac{b_s\sqrt{Lf(m_t)}}{\Tr(H) \sigma_t}\right)
      \label{eq:penalty_end-1}
    \\
    &-  v \cdot \log^+ \left(\frac{\Tr(H)\sigma_t}{b_\ell\sqrt{Lf(m_t)}}\right)
      \label{eq:penalty_end}
      \enspace.
  \end{align}
  To inspect the sum of \eqref{eq:penalty_start} $\sim$ \eqref{eq:penalty_end}, we calculate the products of $\inds, \indl, \indup$, and $\inddown$.
  Note $f(m_t)\geq f(m_{t+1})$ for all $t\geq 0$ in Algorithm~\ref{algo}.
  \begin{align}
    \inds\indup
    &=\ind{ 1\leq \frac{ b_s\sqrt{Lf(m_{t+1})} }{\Tr(H)\aup\sigma_t} }\cdot\indup,
    \\
    \inds\inddown
    &=\ind{ 1\leq \frac{ b_s\sqrt{Lf(m_{t})} }{\Tr(H)\adown\sigma_t} }\cdot\inddown, \\
    &=\left(\ind{ 1\leq \frac{ b_s\sqrt{Lf(m_{t})} }{\Tr(H)\sigma_t} },
      +\ind{ \adown \leq \frac{ b_s\sqrt{Lf(m_{t})} }{\Tr(H)\sigma_t} < 1}\right)\cdot\inddown.
  \end{align}
  \begin{align}
    \indl\indup
    &= \ind{1\leq \frac{\Tr(H)\aup\sigma_t}{b_\ell \sqrt{Lf(m_{t+1})} } }\cdot\indup\\
    &= \ind{1\leq \frac{\Tr(H)\sigma_t}{b_\ell \sqrt{Lf(m_{t+1})} } }\cdot\indup
      + \ind{\frac{1}{\aup}\leq \frac{\Tr(H)\sigma_t}{b_\ell \sqrt{Lf(m_{t+1})} } <1 }\cdot\indup\\
    &= \ind{1\leq \frac{\Tr(H)\sigma_t}{b_\ell \sqrt{Lf(m_{t})} } }\cdot\indup\\
    &+ \ind{\frac{\Tr(H)\sigma_t}{b_\ell \sqrt{Lf(m_{t})} } <1 \leq \frac{\Tr(H)\sigma_t}{b_\ell \sqrt{Lf(m_{t+1})} } }\cdot\indup, \\
    &+ \ind{\frac{1}{\aup}\leq \frac{\Tr(H)\sigma_t}{b_\ell \sqrt{Lf(m_{t+1})} } <1 }\cdot\indup
    \\
    \indl\inddown
    &= \ind{1\leq \frac{\Tr(H)\adown\sigma_t}{b_\ell \sqrt{Lf(m_{t})} } }\cdot\inddown
  \end{align}
  Note that
  $\inds\indup=1 \Rightarrow 1\leq \frac{b_s \sqrt{Lf(m_t)}}{\Tr(H)\sigma_t}$
  and
  $\indl\inddown=1 \Rightarrow 1\leq \frac{\Tr(H)\sigma_t}{b_l \sqrt{Lf(m_t)}}$.
  Exploiting these compositions of the indicator functions,  the sum of \eqref{eq:penalty_start}--\eqref{eq:penalty_end} is proved to be non-positive in general as follows:
  \begin{multline}
    \left(v\cdot\log\left(\frac{b_s\sqrt{ Lf(m_{t})}}{\Tr(H) \sigma_{t}}\right)
      - v \cdot \log^+ \left(\frac{b_s\sqrt{Lf(m_t)}}{\Tr(H) \sigma_t}\right)\right)\inds\indup\\
    =0
       \enspace,
    \end{multline}
  \begin{multline}
    \label{eq:penalty_diff_inds_inddown}
    \left(v\cdot\log\left(\frac{b_s\sqrt{ Lf(m_{t})}}{\Tr(H) \sigma_{t}}\right)
      - v \cdot \log^+ \left(\frac{b_s\sqrt{Lf(m_t)}}{\Tr(H) \sigma_t}\right)\right)\inds\inddown\\
    =
       v\cdot\log\left(\frac{b_s\sqrt{ Lf(m_{t})}}{\Tr(H) \sigma_{t}}\right)
       \cdot\ind{ \adown \leq \frac{ b_s\sqrt{Lf(m_{t})} }{\Tr(H)\sigma_t} < 1}
      \cdot\inddown
    \\
    \leq 0
          \enspace,
    \end{multline}
    \begin{multline}
    \label{eq:penalty_diff_indl_indup}
    \left( v\cdot\log\left(\frac{\Tr(H)\sigma_{t}}{b_\ell\sqrt{ Lf(m_{t})}}\right)
      -  v \cdot \log^+ \left(\frac{\Tr(H)\sigma_t}{b_\ell\sqrt{Lf(m_t)}}\right)\right)\indl\indup\\
    =
        \left( v\cdot\log\left(\frac{\Tr(H)\sigma_{t}}{b_\ell\sqrt{ Lf(m_{t})}}\right)
            - v \cdot \log^+ \left(\frac{\Tr(H)\sigma_t}{b_\ell\sqrt{Lf(m_t)}}\right)
            \right)
          \\
          \quad \cdot
          \Biggl(
          \ind{\frac{\Tr(H)\sigma_t}{b_\ell \sqrt{Lf(m_{t})} } <1 \leq \frac{\Tr(H)\sigma_t}{b_\ell \sqrt{Lf(m_{t+1})} } }
          \\
          \quad + \ind{\frac{1}{\aup}\leq \frac{\Tr(H)\sigma_t}{b_\ell \sqrt{Lf(m_{t+1})} } <1 }
          \Biggr)
            \cdot\indup
    \\
    \leq 0
          \enspace,
  \end{multline}

    \begin{multline}
    \left( v\cdot\log\left(\frac{\Tr(H)\sigma_{t}}{b_\ell\sqrt{ Lf(m_{t})}}\right)
      -  v \cdot \log^+ \left(\frac{\Tr(H)\sigma_t}{b_\ell\sqrt{Lf(m_t)}}\right)\right)\indl\inddown
    \\
    =0
       \enspace.
  \end{multline}
  Note that \eqref{eq:penalty_end-1} and \eqref{eq:penalty_end} are non-positive.
  Subsequently,
  \begin{align}
    V(\theta_{t+1}) - V(\theta_t)
    \leq& \left(1-(\inds-\indl)\frac{v}{2}\right)\cdot\log\left(\frac{f(m_{t+1})}{f(m_t)}\right) \\
    &+v\cdot\log\left(\frac{\aup}{\adown}\right)\cdot\inds\inddown
      -v\cdot\log\left(\aup\right)\cdot\inds, \\
    &+v\cdot\log\left(\frac{\aup}{\adown}\right)\cdot\indl\indup
      +v\cdot\log\left(\adown\right)\cdot\indl\\
    \leq& \left(1-\frac{v}{2}\right)\cdot\log\left(\frac{f(m_{t+1})}{f(m_t)}\right)
    +v\cdot\log\left(\frac{\aup}{\adown}\right)
    \label{eq:hold_above_in_general_case}
      \enspace.
  \end{align}
  In light of \Cref{lemma:cases}, we have
  \begin{equation}
    \E\left[ \log\left(\frac{f(m_{t+1})}{f(m_t)}\right)  \mid \mathcal{F}_t \right]
    \leq - w \enspace.
  \end{equation}
  Taking the expectation on both sides of the inequality above, we obtain
  \begin{align}
    \MoveEqLeft[2]
    \E[ V(\theta_{t+1}) - V(\theta_t)\mid \mathcal{F}_t ]
    \\
 &\leq - \left(1-\frac{v}{2}\right) \cdot w + v\cdot\log\left(\frac{\aup}{\adown}\right)
    \\
 &\leq - \frac12 \cdot w + v\cdot\log\left(\frac{\aup}{\adown}\right)
    \\
 &= - \frac12 \cdot w + \min \left\{ \frac{w}{4}, \log\left(\frac{\aup}{\adown}\right) \right\}
    \\
 &\leq - \frac{w}{4} \enspace.
  \end{align}
  This completes the proof.
\end{proof}
\hrule

\subsection{Proof of \Cref{lemma:v_variance_bound}}\label{apdx:lemma:v_variance_bound}
    \begin{proof}

    \begin{equation}
        \Var[V(\theta_{t+1})\mid\mathcal{F}_t]
        \leq
        \E[\left(V(\theta_{t+1}) - V(\theta_t)\right)^2\mid\mathcal{F}_t]
        \enspace
    \end{equation}
     holds. Next, we prove that
  \begin{align}
    V_{t+1} - V_t \leq v \cdot \log\left(\frac{\aup}{\adown}\right)
    \label{eq:constant_upper_bound_of_V}
  \end{align}
  \noindent
  and
  \begin{multline}
    V(\theta_{t+1}) - V(\theta_t)
    \\
    \geq
    \left(1+v\right)\cdot\log\left(\frac{f(m_{t+1})}{f(m_t)}\right)
      -2v\cdot\log\left(\aup\right)
      +v\cdot\log\left(\adown\right)
      \enspace.
    \label{eq:constant_lower_bound_of_V}
    \end{multline}
  First, the discussion in \Cref{apdx:lemma:reasonable_sigma} until \eqref{eq:hold_above_in_general_case} holds for any $m_t$ and $\sigma_t$.
  Therefore, upper bound \eqref{eq:constant_upper_bound_of_V} is straightforward from \eqref{eq:hold_above_in_general_case}.

  Next, we use the lower bounds of
  \eqref{eq:penalty_diff_inds_inddown} and \eqref{eq:penalty_diff_indl_indup}
  to prove the lower bound \eqref{eq:constant_lower_bound_of_V}.
  From the discussion in \Cref{apdx:lemma:reasonable_sigma} until \eqref{eq:penalty_end}, for any $\theta_t$,

  \begin{align}
    V(\theta_{t+1}) - V(\theta_t)
    =&
    \label{eq:main_terms_of_V_diff_start_re}
    \left(1-(\inds-\indl)\frac{v}{2}\right)\cdot\log\left(\frac{f(m_{t+1})}{f(m_t)}\right) \\
    &+v\cdot\log\left(\frac{\aup}{\adown}\right)\cdot\inds\inddown
      -v\cdot\log\left(\aup\right)\cdot\inds, \\
    \label{eq:main_terms_of_V_diff_end_re}
    &+v\cdot\log\left(\frac{\aup}{\adown}\right)\cdot\indl\indup
      +v\cdot\log\left(\adown\right)\cdot\indl\\
      \label{eq:penalty_start_re}
    &+ v\cdot\log\left(\frac{b_s\sqrt{ Lf(m_{t})}}{\Tr(H) \sigma_{t}}\right)\cdot\inds
    \\
    &+ v\cdot\log\left(\frac{\Tr(H)\sigma_{t}}{b_\ell\sqrt{ Lf(m_{t})}}\right)\cdot\indl\\
    &- v \cdot \log^+ \left(\frac{b_s\sqrt{Lf(m_t)}}{\Tr(H) \sigma_t}\right)
    \\
    &-  v \cdot \log^+ \left(\frac{\Tr(H)\sigma_t}{b_\ell\sqrt{Lf(m_t)}}\right)
      \label{eq:penalty_end_re}
      \enspace.
  \end{align}
  To make a lower bound of the expression above, we use the lower bounds of \eqref{eq:penalty_diff_inds_inddown} and \eqref{eq:penalty_diff_indl_indup}.
  \eqref{eq:penalty_diff_inds_inddown}:
  \begin{align}
    &\left(v\cdot\log\left(\frac{b_s\sqrt{ Lf(m_{t})}}{\Tr(H) \sigma_{t}}\right)
      - v \cdot \log^+ \left(\frac{b_s\sqrt{Lf(m_t)}}{\Tr(H) \sigma_t}\right)\right)\inds\inddown\\
    =&
       v\cdot\log\left(\frac{b_s\sqrt{ Lf(m_{t})}}{\Tr(H) \sigma_{t}}\right)
       \cdot\ind{ \adown \leq \frac{ b_s\sqrt{Lf(m_{t})} }{\Tr(H)\sigma_t} < 1}
       \cdot\inddown
    \\
    \geq& v\cdot\log(\adown)
  \end{align}

  \eqref{eq:penalty_diff_indl_indup}:
  \begin{align}
    &\left( v\cdot\log\left(\frac{\Tr(H)\sigma_{t}}{b_\ell\sqrt{ Lf(m_{t})}}\right)
      -  v \cdot \log^+ \left(\frac{\Tr(H)\sigma_t}{b_\ell\sqrt{Lf(m_t)}}\right)\right)\indl\indup\\
    =&
          \notag
        \left( v\cdot\log\left(\frac{\Tr(H)\sigma_{t}}{b_\ell\sqrt{ Lf(m_{t})}}\right)
            - v \cdot \log^+ \left(\frac{\Tr(H)\sigma_t}{b_\ell\sqrt{Lf(m_t)}}\right)
            \right)
          \\
          \notag
          &\cdot
          \Biggl(
          \ind{\frac{\Tr(H)\sigma_t}{b_\ell \sqrt{Lf(m_{t})} } <1 \leq \frac{\Tr(H)\sigma_t}{b_\ell \sqrt{Lf(m_{t+1})} } }
          \\
            \label{eq:penalty_diff_indl_indup_re}
          &+ \ind{\frac{1}{\aup}\leq \frac{\Tr(H)\sigma_t}{b_\ell \sqrt{Lf(m_{t+1})} } <1 }
          \Biggr)
          \cdot\indup
    \end{align}
    Further, the lower bounds on each indicator function of \eqref{eq:penalty_diff_indl_indup_re} are
    \begin{align}
        &\left( v\cdot\log\left(\frac{\Tr(H)\sigma_{t}}{b_\ell\sqrt{ Lf(m_{t})}}\right)
            - v \cdot \log^+ \left(\frac{\Tr(H)\sigma_t}{b_\ell\sqrt{Lf(m_t)}}\right)
            \right)
          \\
          &\cdot
          \ind{\frac{\Tr(H)\sigma_t}{b_\ell \sqrt{Lf(m_{t})} } <1 \leq \frac{\Tr(H)\sigma_t}{b_\ell \sqrt{Lf(m_{t+1})} } }
          \cdot\indup
    \\
    \geq&
        v\cdot\log\left(\frac{\Tr(H)\sigma_{t}}{b_\ell\sqrt{ Lf(m_{t})}}\right)
        \\
          &\cdot
          \ind{ \frac{\Tr(H)\sigma_t}{b_\ell \sqrt{Lf(m_{t})} } <1 \leq \frac{\Tr(H)\sigma_t}{b_\ell \sqrt{Lf(m_{t+1})} } }
    \\
    \geq&
        v\cdot\log\left(\sqrt{ \frac{f(m_{t+1})}{f(m_t)} } \cdot\frac{\Tr(H)\sigma_{t}}{b_\ell\sqrt{ Lf(m_{t+1})}}\right)
        \\
          &\cdot
          \ind{ \frac{\Tr(H)\sigma_t}{b_\ell \sqrt{Lf(m_{t})} } <1 \leq \frac{\Tr(H)\sigma_t}{b_\ell \sqrt{Lf(m_{t+1})} } }
    \\
    \geq&
        v\cdot\log\left(\sqrt{ \frac{f(m_{t+1})}{f(m_t)} }\right)
    \enspace,
  \end{align}
  and
    \begin{align}
        &\left( v\cdot\log\left(\frac{\Tr(H)\sigma_{t}}{b_\ell\sqrt{ Lf(m_{t})}}\right)
            - v \cdot \log^+ \left(\frac{\Tr(H)\sigma_t}{b_\ell\sqrt{Lf(m_t)}}\right)
            \right)
          \\
          &\ind{\frac{1}{\aup}\leq \frac{\Tr(H)\sigma_t}{b_\ell \sqrt{Lf(m_{t+1})} } <1 }
          \cdot\indup
    \\
    \geq&
        v\cdot\log\left(\frac{1}{\aup}\right)
        \enspace.
  \end{align}
  Therefore, the lower bound of the sum of \cref{eq:penalty_start_re}$\sim$\eqref{eq:penalty_end_re} is
  \begin{align}
      v\cdot\log\left(\sqrt{ \frac{f(m_{t+1})}{f(m_t)} }\right) - v\cdot\log\left(\aup\right)
      \label{eq:lower_bound_penalty}
  \enspace.
  \end{align}
  Additionally, the lower bound of the sum of the remaining terms \eqref{eq:main_terms_of_V_diff_start_re} $\sim$ \eqref{eq:main_terms_of_V_diff_end_re} is
  \begin{align}
    \log\left(\frac{f(m_{t+1})}{f(m_t)}\right)
      -v\cdot\log\left(\aup\right)
      +v\cdot\log\left(\adown\right)
      \label{eq:lower_bound_main_terms}
      \enspace.
  \end{align}
  With \eqref{eq:lower_bound_penalty} and \eqref{eq:lower_bound_main_terms},
  we see that for any $m_t$ and $\sigma_t$, \eqref{eq:constant_lower_bound_of_V} holds.
  Integrating \eqref{eq:constant_lower_bound_of_V} and \eqref{eq:constant_upper_bound_of_V},
  we obtain
  \begin{align}
  &|V(\theta_{t+1}) - V(\theta_t)|
  \\
  \leq&
  \max
    \Biggl\{
  v \cdot \log\left(\frac{\aup}{\adown}\right),
  \\
  &\left(1+v\right)\cdot\left|\log\left(\frac{f(m_{t+1})}{f(m_t)}\right)\right|
    +2v\cdot\log\left(\aup\right)
    -v\cdot\log\left(\adown\right)
    \Biggr\}
    \\
    \leq&
  \left(1+v\right)\cdot\left|\log\left(\frac{f(m_{t+1})}{f(m_t)}\right)\right|
    +2v\cdot\log\left(\aup\right)
    -v\cdot\log\left(\adown\right)
    \enspace.
    \end{align}

    As the only random variable in the RHS of the expression above is $\left|\log\left(\frac{f(m_{t+1})}{f(m_t)}\right)\right|$,
    to prove $\E[\left(V(\theta_{t+1}) - V(\theta_t)\right)^2\mid\mathcal{F}_t]$, we need to prove that
    \begin{align}
        \E\left[\left|\log\left(\frac{f(m_{t+1})}{f(m_t)}\right)\right|\mid\mathcal{F}_t\right]
        <\infty
        \enspace,
        \label{eq:abs_qualitygain_integrable}
    \end{align}
    and
    \begin{align}
        \E\left[\left(\log\left(\frac{f(m_{t+1})}{f(m_t)}\right)\right)^2\mid\mathcal{F}_t\right]
        <\infty
        \enspace.
        \label{eq:squared_qualitygain_integrable}
    \end{align}
   By exploiting the fact that $x \leq \exp(x) - 1$ and $x^2 \leq 2 (\exp(x) - 1)$,
   both \eqref{eq:abs_qualitygain_integrable} and \eqref{eq:squared_qualitygain_integrable} are straightforward for $d>3$ with \Cref{lemma:variancebound}.
   Namely,
    \begin{align}
        \E\left[\left|\log\left(\frac{f(m_{t+1})}{f(m_t)}\right)\right|\mid\mathcal{F}_t\right]
        \leq\frac{1}{d-3}\cdot\frac{U}{L}
        <\infty
        \enspace,
    \end{align}
    and
    \begin{align}
        \E\left[\left(\log\left(\frac{f(m_{t+1})}{f(m_t)}\right)\right)^2\mid\mathcal{F}_t\right]
        \leq\frac{2}{d-3}\cdot\frac{U}{L}
        <\infty
        \enspace.
    \end{align}
    This completes the proof.
\end{proof}

\end{document}